\documentclass{alt2019}

\usepackage{times}
\usepackage{hyperref,url}
\usepackage{algorithm}
\usepackage{algorithmic}
\usepackage{amsmath}
\usepackage{color}
\usepackage{colortbl}
\usepackage{graphicx}
\usepackage{epstopdf}
\usepackage{multirow}
\usepackage{psfrag}

\newcommand{\w}{\ensuremath{\mathbf{w}}}
\newcommand{\x}{\ensuremath{\mathbf{x}}}
\newcommand{\y}{\ensuremath{\mathbf{y}}}

\newcommand{\0}{\ensuremath{\mathbf{0}}}

\newcommand{\bbE}{\ensuremath{\mathbb{E}}}

\newcommand{\calA}{\ensuremath{\mathcal{A}}}

\newcommand{\calO}{\ensuremath{\mathcal{O}}}

\newcommand{\abs}[1]{\left\lvert#1\right\rvert}
\newcommand{\norm}[1]{\left\lVert#1\right\rVert}

\newcommand{\caja}[4][1]{{%
    \renewcommand{\arraystretch}{#1}%
    \begin{tabular}[#2]{@{}#3@{}}%
      #4%
    \end{tabular}%
    }}

{%
\begin{list}{#1}{
\vspace{-\topsep}
\vspace{-\partopsep}
\setlength{\itemindent}{0cm}
\setlength{\rightmargin}{0cm}
\setlength{\listparindent}{0cm}
\settowidth{\labelwidth}{#1}
\setlength{\leftmargin}{\labelwidth}
\addtolength{\leftmargin}{\labelsep}
\setlength{\itemsep}{0cm}
}%
}%
{%
\end{list}
\vspace{-\topsep}
\vspace{-\partopsep}
}

{\begin{enumerate}%
\renewcommand{\theenumi}{\roman{enumi}}%
\renewcommand{\labelenumi}{(\theenumi)}}%
{\end{enumerate}}

\newtheorem{thm}{Theorem}
\newtheorem{lem}[thm]{Lemma}
\newtheorem{cor}[thm]{Corollary}
\newtheorem{rmk}[thm]{Remark}
\definecolor{light-gray}{gray}{0.8}

\DeclareMathOperator*{\argmin}{arg\,min}

\newcommand{\sbr}[1]{\left[#1\right]}
\newcommand{\rbr}[1]{\left(#1\right)}
\newcommand{\cbr}[1]{\left\{#1\right\}}
\newcommand{\dotp}[2]{\langle{#1},\,{#2}\rangle}

\newcommand{\ie}{i.e.\@}
\newcommand{\eg}{e.g.\@}

\newcommand{\tw}{\tilde{\w}}
\newcommand{\hF}{\hat{F}}
\newcommand{\hG}{\hat{G}}
\newcommand{\hw}{\hat{\w}}

\newcommand{\hx}{\hat{\x}}
\newcommand{\hphi}{\hat{\phi}}

\usepackage{mathtools}

\newcommand{\RSAG}{\operatorname{RSAG}}
\newcommand{\MP}{\operatorname{MP}}
\newcommand{\MPAGD}{\operatorname{MP-AGD}}
\newcommand{\MPSVRG}{\operatorname{MP-SVRG}}

\usepackage{eqparbox}

\title[Stochastic Nonconvex Optimization with Large Minibatches]
{Stochastic Nonconvex Optimization with Large Minibatches}

\coltauthor{\Name{Weiran Wang} \Email{weiranw@amazon.com} \\
  \addr Amazon Alexa \\
  101 Main St \\
  Cambridge, MA 02142, USA
  \AND
  \Name{Nathan Srebro} \Email{nati@ttic.edu} \\
  \addr Toyota Technological Institute at Chicago \\
  6045 S Kenwood Ave \\
  Chicago, IL 60637, USA
 }

\begin{document}

\maketitle

\begin{abstract}
We study stochastic optimization of nonconvex loss functions, which are typical objectives for training neural networks. We propose stochastic approximation algorithms which optimize a series of regularized, nonlinearized losses on large minibatches of samples, using only first-order gradient information. Our algorithms provably converge to an approximate critical point of the expected objective with faster rates than minibatch stochastic gradient descent, and facilitate better parallelization by allowing larger minibatches. 
\end{abstract}


\section{Introduction}
\label{sec:intro}

Machine learning algorithms ultimately try to optimize the performance of models in the population. 
Consider the following stochastic optimization (generalized learning) problem~\citep{Vapnik00a,Shalev_09a}: 
\begin{align} \label{e:obj}
  \min_{\w}\; \phi (\w) := \bbE_{\xi \sim D} \left[ \ell (\w, \xi) \right]
\end{align}
where our goal is to learn a predictor $\w$ given the instantaneous (loss) function $\ell(\w,\xi)$ and i.i.d. samples $\xi_1,\xi_2,\dots$ from some unknown data distribution $D$. 
In this work, we focus on losses $\ell (\w, \xi)$ that are \emph{nonconvex} functions of $\w$; a prevalent example of this setting is the training of deep neural networks, where $\w$ denotes the collection of trainable weights of a deep learning model, and $\ell(\w,\xi)$ measures the loss of prediction (\eg, classification or regression) for the sample $\xi$ using weights $\w$.

Despite efforts in introducing higher order optimization methods to this problem, stochastic gradient descent (SGD) with minibatches and its variants~\citep{Bottou91a,Lecun_98b,Duchi_11a,Zeiler12a,KingmaBa15a} remain by far the most popular methods for training deep neural networks, due to its simplicity and superior performance than the alternatives. In the vanilla version of minibatch SGD, we compute the averaged gradient over a small set of samples (called minibatch), e.g. using the backpropagation algorithm, and simply take a step in the negative direction. The use of minibatch (as opposed to a single sample for estimating the gradient) makes the training process more stable as it reduces the variance of gradient estimate. Moreover, to process the same amount of samples, it takes smaller number of updates if a larger minibatch size is used, and the backpropagation procedure on a larger minibatch can utilize massive parallelization of linear algebra routines provided by advanced computational hardware (GPUs and clusters). As a result, using a larger minibatch in SGD can potentially significantly reduce the parallel training time.

Recently, there has been some empirical analysis of minibatch SGD for non-convex problems, focusing mostly on practical issues such as the correct scaling of stepsize (learning rate) and momentum with minibatch size~\citep{Goyal_17a,Hoffer_17a}. A prominent observation is that, by properly setting the stepsize, minibatch SGD works well (converges to similarly good test set performance) for a wide range of minibatch sizes, while large minibatches facilitate better parallelization. But intriguingly, beyond certain threshold of minibatch size, training result start to deteriorate and simply scaling the stepsize does not help. The primary goal of this work is to theoretically investigate the issue of minibatch size in stochastic nonconvex optimization, and to provide practical algorithms/guidance to training deep neural networks.
Our analysis applies to smooth nonconvex instantaneous losses, and uses a characterization of nonconvex functions named almost-convexity, which we introduce below.

\paragraph{Problem setup}
In this work, we assume that the differentiable instantaneous loss
$\ell(\w,\xi)$ is $\beta$-smooth and $\sigma$-almost convex. Recall that a function $f(\w)$ is $\beta$-smooth if $\norm{\nabla f (\w) - \nabla f (\w^\prime)} \le \beta \norm{\w - \w^\prime}$ for all $\w, \w^\prime$, and in this case we have the following quadratic approximation for $f (\w)$:
\begin{align*}
  \abs{ f(\w) - f(\w^\prime) - \dotp{\nabla f(\w^\prime)}{\w - \w^\prime} } \le \frac{\beta}{2} \norm{\w-\w^\prime}^2, \qquad 
  \forall \w,\w^\prime.
\end{align*}
On the other hand, a nonconvex function $f(\w)$ is $\sigma$-almost convex for $\sigma \ge 0$ if
\begin{align*}
  f(\w) - f(\w^\prime) - \dotp{\nabla f(\w^\prime)}{\w - \w^\prime} \ge - \frac{\sigma}{2} \norm{\w-\w^\prime}^2, \qquad 
  \forall \w,\w^\prime.
\end{align*}
A convex function is $0$-almost convex, and a $\beta$-smooth function is $\sigma$-almost convex for some $\sigma \le \beta$. For a twice differentiable function that is both $\beta$-smooth and $\sigma$-almost convex, the eigenvalues of its Hessian matrix lie in $[-\sigma,\beta]$. 

We now discuss some properties we need when accessing the stochastic objective through samples. 
First, the stochastic gradient estimated on a single sample is unbiased, \ie,
\begin{align*}
  \bbE \sbr{ \nabla \ell (\w,\xi) } = \nabla \phi (\w), \qquad \forall \w.
\end{align*}
Second, as is common in the stochastic optimization literature (see, \eg, \citealp{Lan12a,GhadimLan16a}), we assume 
and that the variance of the stochastic gradient is bounded by $V^2$, \ie,
\begin{align*}
  \bbE_{\xi} \norm{\nabla \ell (\w,\xi) - \nabla \phi (\w) }^2 \le V^2, \qquad \forall \w. 
\end{align*}
Denote by $\phi^* = \min_{\w} \phi(\w)$ the (globally) minimum value of $\phi(\w)$, which we assume to be finite. Since in general we can not hope to efficiently obtain the global minimum of a nonconvex objective, %~\citep{NemirovYudin83a}, 
the reasonable goal here is to find an approximate critical point $\w$ satisfying for some $\varepsilon>0$ that
\begin{align} \label{e:critical}
  \bbE \norm{\nabla \phi (\w)}^2 \le \varepsilon^2.
\end{align}
We are interested in the number of samples and the amount of computation needed to achieve this goal.

\paragraph{Significance of almost convexity}
One may wonder whether it is reasonable to assume the nonconvex objective to be almost
convex. We note that, almost convexity arises from the optimization of
general smooth nonconvex objectives. Based on the Hessian Lipschitz assumption,
\citet{Carmon_17a} have shown that one can alternate over the negative curvature
descent algorithm (which eventually leads us to a point at which the Hessian has
small negative eigenvalues) and optimizing almost-convex problems, to
obtain overall faster convergence than gradient descent in the
non-stochastic setting. 
In fact, the almost-convex procedure is shown to be the key to the faster rate; see
their Section 4 and also~\citet[Appendix A]{Allen-Zhu17a}. These
results motivate us to study the stochastic version of the almost-convex
problems, under common assumptions used to analyze stochastic gradient
descent. We verify that improvement obtained in the non-stochastic
case does carry over to the stochastic case, and in turn facilitates
better parallelism. 
On the other hand, while we assume above that each individual loss is
almost convex, relaxation to the almost convexity of only the population
objective will be discussed later. 

\subsection{Minibatch SGD for nonconvex stochastic optimization}
\label{sec:rsag}

The theoretical performance of minibatch SGD has been relatively well studied for convex objectives~\citep{Lan12a,Dekel_12a,Cotter_11a}. After $T$ minibatch gradient updates, each using a stochastic gradient estimated on $b$ samples, the accelerated minibatch SGD algorithm on a convex $\phi(\w)$ returns an iterate $\w$ satisfying 
\begin{align} \label{e:acc-minibatch-sgd-convex}
  \bbE \sbr{ \phi (\w) - \phi (\w_*) } \le \calO \rbr{ \frac{\beta \norm{\w_0 - \w_*}^2}{T^2} +  \frac{V \norm{\w_0 - \w_*}}{\sqrt{b T}} },
\end{align}
where $\w_0$ is the initialization, and $\w_*=\argmin_{\w} \phi(\w)$.

For stochastic nonconvex optimization,~\citet{Ghadim_16a} analyzed the convergence of minibatch SGD under the same problem setup as ours, and~\citet{GhadimLan16a} further proposed a randomized stochastic accelerated gradient (RSAG) method which resembles the accelerated stochastic approximation method for convex optimization~\citep{Lan12a}. After $T$ minibatch gradient updates, each using stochastic gradient estimated on $b$ samples, their algorithms return an iterate $\w$ satisfying\footnote{This can be deduced from eqn~(3.20) of~\citet{GhadimLan16a}, as the variance of stochastic gradient reduces to $V^2 / b$ when $b$ samples are used in estimating the gradient.}
\begin{align}\label{e:rsag}
  \bbE \norm{\nabla \phi (\w) }^2 
  \le \calO \rbr{ \frac{\beta (\phi (\w_0) - \phi^*)}{T} + \frac{V \sqrt{ \beta (\phi (\w_0) - \phi^*)}}{\sqrt{bT}} }.
\end{align}
To parse this result, we observe the following:
\begin{itemize}
\item When $V=0$, or in other words exact gradients are used, the
  second term vanishes and the convergence rate reduces to
  $\norm{\nabla \phi (\w) }^2 \le \calO \rbr{ \frac{\beta (\phi (\w_0)
      - \phi^*)}{T} }$, recovering the rate for 
  deterministic gradient descent~\citep{Ghadim_16a}. We refer to this
  term as the ``optimization error'' since it is independent of the
  samples.
\item  The second term in~\eqref{e:rsag} results from the noise in stochastic gradients, and we refer to it as the ``sample error'' since it results from the sampling process. This term is asymptotically dominant as long as $b=\calO \rbr{ \frac{V^2 T}{\beta (\phi (\w_0) - \phi^*)} }$. 
  Using a much larger $b$, the first term of~\eqref{e:rsag} becomes dominant but since the first term is independent of $b$, the algorithm is no longer sample efficient (it is using more fresh samples than needed). 
  This is consistent with the empirical findings of practitioners of minibatch SGD: beyond certain minibatch size, learning slows down and in particular, the objective on test set (which is an estimate of the population objective) do not decrease faster with the amount of computation, even though more samples and computation (\eg, backpropagation) is involved in each stochastic gradient update. 
\end{itemize}

Denote the total number of samples used by $N=b T$. The convergence rate~\eqref{e:rsag} indicates that, to find an critical point satisfying~\eqref{e:critical}, the total sample needed is $N(\varepsilon)=\calO \rbr{ \frac{V^2 \beta (\phi (\w_0) - \phi^*)}{\varepsilon^4} }$, while the maximum minibatch size that maintains sample efficiency, and the iteration complexity using this minibatch size, are respectively
\begin{align} \label{e:rsag-complexity}
  b_{\RSAG}=\calO \rbr{ \frac{V \sqrt{N(\varepsilon)}}{\sqrt{\beta (\phi (\w_0) - \phi^*)}} } = \calO \rbr{\frac{V^2}{\varepsilon^2}},
  \qquad 
  T_{\RSAG}=\calO \rbr{ \frac{\beta (\phi(\w_0)-\phi^*)}{\varepsilon^2} }.
\end{align}
An optimal choice of minibatch size (up to constants) for this method
is thus $b_{\RSAG}$---in the regmime $b < b_{\RSAG}$, increasing the
minibatch size leads to reduced number of iterations, down to a
minimum of $T_{\RSAG}$, but any further increase would increase the
overall work performed (overall number of vector operations), without
decreasing the required number of iterations.  

In a highly parallel setting, the number of mini-batch gradient
evaluations $T$ captures the {\em parallel runtime}, since each such
evaluation can be efficiently parallelized over the machines, and so
throughout we account of the parallel runtime in terms of the number
of such mini-batch gradient evaluations, regardless of the number of
points involved in each such mini-batch.  We also account for the
overall work performed (or energy consumed), in terms of the overall
number of gradient evaluations or vector operations.  For mini-batch
SGD we process each point once, and so the overall work is
$\calO(N(\varepsilon))$.

\subsection{Iterative convexification methods for deterministic
  nonconvex optimization}
\label{sec:erm}

The approach discussed in Section~\ref{sec:rsag} draws fresh samples
in each update.  A alternative approach is to draw a single set of $n$
training examples and minimize the empirical risk (or sample averaged
approximation) on these points using deterministic optimization techniques:
\begin{align} \label{e:obj-erm}
  \min_{\w}\; \hphi(\w) = \frac{1}{n} \sum_{i=1}^n \ell(\w,\xi_i). 
\end{align}
We could attempt to optimize \eqref{e:obj-erm} using gradient descent.
But recently,~\citet{Carmon_17a} and \citet{Allen-Zhu17a} demonstrated
that when $\sigma \ll \beta$, \ie, when the objective is not too
nonconvex (the negative eigenvalues are not too large in magnitude),
one can obtain faster convergence than gradient descent, by
transforming an almost-convex objective into a series of (strongly)
convex optimization problems which are then solved very efficiently by
batch accelerated gradient descent; %% or finite-sum methods;
this technique will be discussed in detail in
Section~\ref{sec:reduction}.  Applying this technique to
\eqref{e:obj-erm}, to obtain a $\w$ satisfying $\norm{ \nabla
  \hphi(\w) }^2 \le \varepsilon^2$, it is sufficient to
perform\footnote{This amounts to plugging in $\gamma=\sigma$
  in~\citet{Carmon_17a}[Lemma~3.1] for their
  \texttt{Almost-Convex-AGD} algorithm. Note that the authors also
  showed the $\tilde{\calO} \rbr{\frac{1}{\varepsilon^{7/4}}}$
  iteration complexity is achievable with second order information.}
\begin{align} \label{e:carmon}
  T_{\textrm{BATCH}} = \tilde{\calO} \rbr{\frac{\sqrt{\sigma \beta} (\hphi (\w_0) - \min\limits_{\w} \hphi (\w) ) }{\varepsilon^2}}
\end{align}
exact (batch) gradient-based updates.  Since again each such batch
computation can be calculated in parallel, we get a parallel runtime
of $\calO(T_{\textrm{BATCH}})$ and the total number of gradient
evaluations required is $\calO(T_{\textrm{BATCH}} \cdot
N(\varepsilon))$.

We suggest algorithms that are based on the same intuition and use the
same convexification procedure, but we propose tackling the stochastic
optimization problem directly, using approximate gradients obtained
from minibatches at each iteration.  This has two advantages.  First,
we are tackling the stochastic optimization objective directly, which
is many cases is our true objective, and so obtain guarantees directly
on the population $\norm{\nabla \phi(\w)}$ rather than merely its
empirical approximation $\norm{\nabla \hphi(\w)}$.  Second, as we
shall see, but using only part of the data at each iteration, instead
of the entire data set, we can reduce the total amount of work
performed (total number of gradient computations and vector
operations) without sacrificing the parallel runtime and accuracy guarantees.

\begin{figure}[t]
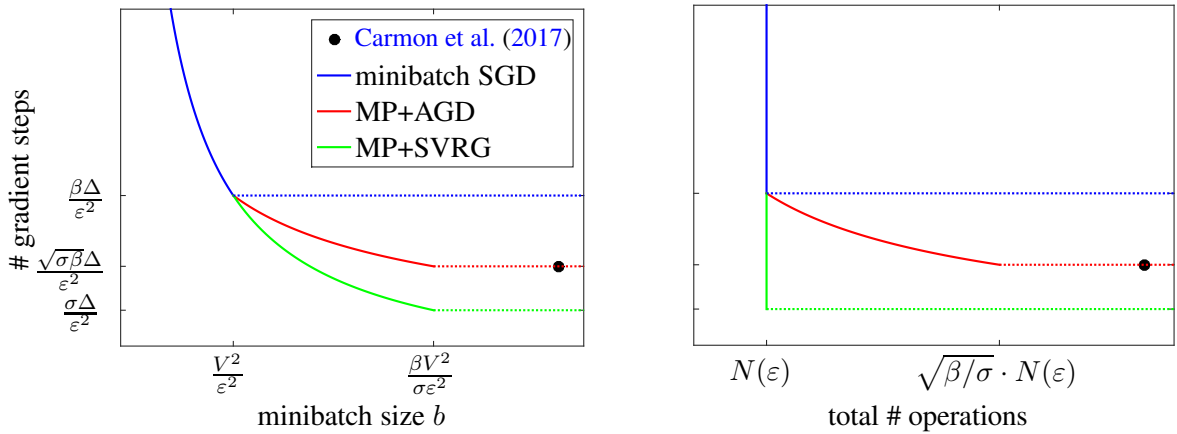

  \centering
  \psfrag{batch}[l][l][0.9]{\citet{Carmon_17a}}
  \psfrag{minibatch SGD}[l][l]{minibatch SGD}
  \psfrag{MP-AGD}[l][l]{MP+AGD}
  \psfrag{MP-SVRG}[l][l]{MP+SVRG}
  \begin{tabular}{@{}c@{\hspace{0.095\linewidth}}c@{}}
    \psfrag{b}[][]{minibatch size $b$}
    \psfrag{GD}[b][t]{\# gradient steps}
    \psfrag{x1}[][]{$\frac{V^2}{\varepsilon^2}$}
    \psfrag{x2}[][]{$\frac{\beta V^2}{\sigma \varepsilon^2}$}
    \psfrag{y1}[r][r]{$\frac{\sigma \Delta}{\varepsilon^2}$}
    \psfrag{y2}[r][r]{$\frac{\sqrt{\sigma \beta} \Delta}{\varepsilon^2}$}
    \psfrag{y3}[r][r]{$\frac{\beta \Delta}{\varepsilon^2}$}
    \includegraphics[width=0.485\linewidth]{mpagd1.eps} & 
    \psfrag{GD}[][]{}
    \psfrag{x1}[][]{$N (\varepsilon)$}
    \psfrag{x2}[][]{$\sqrt{\beta/\sigma} \cdot N (\varepsilon)$}
    \psfrag{y1}[][]{}
    \psfrag{y2}[][]{}
    \psfrag{y3}[][]{}
    \psfrag{Energy}[][]{total \# operations}
    \includegraphics[width=0.42\linewidth]{mpagd2.eps}
  \end{tabular}
  \vspace*{-2ex}
  \caption{Illustration of theoretical guarantees for our algorithms (MP equipped with 
    two convex optimizers---AGD and SVRG) and the comparisons with minibatch SGD, in terms 
    of number of gradient steps
    vs. the minibatch size $b$ (left plot), and number of gradient steps
    vs. the total number of vector operations and gradient
    calculations. Here $\Delta = \phi(\w_0) - \phi^*$. We have assumed for MP+SVRG that 
    the parallel runtime for computing the gradient on a minibatch of $b$ samples (with 
    communications) is much larger than computing the gradient on a single example locally, 
    see discussion in Section~\ref{sec:distributed}. 
    The regimes that are not sample-efficient are dotted.}
  \label{f:mpagd}
\end{figure}

\subsection{Our contributions}
\label{sec:contributions}

We propose stochastic approximation algorithms that provably converge
to an approximate critical point of the nonconvex population
objective. In our template algorithm, which we refer to as
``minibatch-prox'' (MP), we draw $b$ fresh samples at each iteration
and approximately optimize a convex objective defined by these
samples:
\begin{align*}
  \w_t \approx \argmin_{\w}\; \frac{1}{b} \sum_{i=1}^b \ell (\w,\xi_i^t) + \frac{\gamma}{2} \norm{\w - \w_{t-1}}^2 \qquad \text{where}\quad \gamma > \sigma,\quad \text{for}\quad t=1,\dots,K.
\end{align*}
Choosing a random iterate $\w$ from the first $K$ iterations of MP, we have that (for large enough $b$)
\begin{align} \label{e:minibatch-prox-rate}
  \bbE \norm{\nabla \phi (\w) }^2 
  \le \calO \rbr{ \frac{\sigma (\phi (\w_0) - \phi^*)}{K} + \frac{V \sqrt{ \beta (\phi (\w_0) - \phi^*)}}{\sqrt{bK}} }.
\end{align}
MP can use larger minibatch while maintaining sample efficiency, at
the cost of more complicated operations (solvings nonlinear
optimization problems) on each minibatch: With a minibatch size of $\Theta \rbr{\frac{\beta V^2}{\sigma \varepsilon^2}}$, MP solves $\calO \rbr{ \frac{\sigma (\phi(\w_0)-\phi^*)}{\varepsilon^2} }$ convex subproblems. 

MP is a meta algorithm and allows us to plug in different optimizers
for solving the convex subproblems on each minibatch. In particular,
when accelerated gradient descent is used as the optimizer, the total
number of gradient steps in MP+AGD with a minibatch of size $\Theta \rbr{\frac{\beta V^2}{\sigma \varepsilon^2}}$ is
\begin{align*}
  T_{\MPAGD} = \tilde{\calO} \rbr{ \frac{ \sqrt{\sigma \beta} (\phi(\w_0)-\phi^*)}{\varepsilon^2} } .
\end{align*}
This significantly improves $T_{\RSAG}$ in~\eqref{e:rsag-complexity}
when $\beta \gg \sigma$, and this is achieved at the cost of a larger
total computational cost (\ie, total number of vector operations and
gradient computations).  The comparison is depicted in
Figure~\ref{f:mpagd}.  In particular, the right panel of the Figure we
see how RSAG and MP+AGD compare in terms of the parallel runtime
(number of gradient steps) and total work (number of individual
gradient computations).  The Pareto optimal point for these two
resources using RSAG is given by $T_{\RSAG}$ and $N(\varepsilon)$
respectively.  MP+AGD does not dominate RSAG, but rather allows us
reduce the parallel runtime at the cost of increasing the total
computation cost, down to a minimum parallel runtime of
$T_{\MPAGD}$.  This is the same parallel runtime
required by the batch methods, but we dominate them since to achieve
this runtime, we still require significantly less total computation.

We also analyze the use of other convex optimizers in MP, and
develop a memory-efficient version for it when the total number of
samples used in each convex subproblem is very large.

Our results have important implications for parallel or distributed
learning: they suggest that it is possible to significantly reduce
parallel runtime, but that this requires going beyond the
minibatch SGD paradigm of using each minibatch only once.

\section{Minibatch-prox (MP) for nonconvex smooth loss}
\label{sec:minibatch}

In this section, we first review the fundamental convexification step in our algorithms which allows us to find approximate critical point by solving convex subproblems, and then propose the basic version of our algorithm and analyze its convergence properties.

\subsection{Convexification of nonconvex problems}
\label{sec:reduction}

The key ingredient in our algorithms is the reduction from the optimization of a nonconvex objective into the optimization of a series of convex problems~\citep{Bertsek79a,Bertsek99a,Carmon_17a,Allen-Zhu17a}. 
%% [pp. 483] of Bertsek99a
Consider the following iterative procedure: for $t=1,\dots,K$,
\begin{align} \label{e:reduction}
  \w_t \approx \w_t^* = \argmin_{\w}\; F_t (\w)  \quad{where}\quad F_t (\w):=\phi(\w)+\frac{\gamma}{2} \norm{\w - \w_{t-1}}^2
\end{align}
where $\gamma > \sigma$. 
At each iteration of this algorithm, one approximately minimizes a regularized objective, where the $\ell_2$ proximity term encourages the new iterate to be close to the previous iterate. With the regularization term, the objective $F_t(\w)$ is $(\gamma - \sigma)$-strongly convex and the global minimizer $\w_t^*$ is unique.  Similar procedures have also been used when $\phi(\w)$ is convex to speedup first order methods~\citep{Lin_15a}.

%% Convex objectives, such as $F_t (\w)$, $t=1,\dots$, are widely recognized to be easier to solve than nonconvex ones. Moreover, 
We can quantify the number of convex subproblems to be solved in~\eqref{e:reduction}, so as to find an approximate critical point of $\phi(\w)$. 
Assume for now that we always obtain the exact minimizer of the subproblem~\eqref{e:reduction} at each iteration $t$, \ie, $\w_t=\w_t^*$. Then we have by the first order optimality of $\w_t^*$ that $\nabla \phi(\w_t) = \gamma (\w_{t-1} - \w_t)$. And in view of this optimality condition, $\w_t$ is an approximate critical point if $\norm{\w_{t-1} - \w_t} \le \varepsilon / \gamma$. On the other hand, if $\norm{\w_{t-1} - \w_t} > \varepsilon / \gamma $, we expect to achieve large reduction in $\phi(\w)$: owing to the strong convexity of $F_t(\w)$, we have $F (\w_{t-1}) - F (\w_t) \ge \frac{\gamma - \sigma}{2} \norm{\w_{t-1}-\w_t}^2$, which is equivalent to
\begin{align*}
  \phi (\w_{t-1}) - \phi (\w_t) \ge \frac{2 \gamma - \sigma}{2} \norm{\w_{t-1} - \w_t}^2
  >  \frac{\varepsilon^2}{2 \gamma}.
\end{align*}
We can not keep decreasing the objective in this way for more than $\frac{2 \gamma \rbr{\phi (\w_0) - \phi^*}}{\varepsilon^2}$ iterations.

The following simple lemma makes this intuition more precise when we have an approximate minimizer at each iteration. A similar and more general result (which applies to constrained optimization/composite objectives) can be found in~\citet[Lemma~4.1]{Allen-Zhu17a}.
\begin{lem}\label{lem:stopping}
  Let $F_t(\w)$ be defined in~\eqref{e:reduction}. If we apply a possibly randomized algorithm $\calA$ to obtain an approximate minimizer $\w_t$ satisfying
  \begin{align} \label{e:error-at-iteration-t}
    \bbE_{\calA} \sbr{ F_t (\w_t) - F_t (\w_t^*) \,|\, \w_{t-1} } \le \epsilon, 
  \end{align}
  %%where the expectation is taken over all the randomness in $\calA$ at iteration $t$. 
  then we have 
  \begin{align} \label{e:gradient-at-iteration-t}
    \bbE_{\calA} \norm{ \nabla \phi (\w_t) }^2 \le 2 \gamma^2 \bbE_{\calA} \norm{\w_{t-1} - \w_t}^2 + 4 (\beta + \gamma) \epsilon.
  \end{align}
\end{lem}

  \begin{cor} \label{cor:reduction}
     If we run the iterative procedure~\eqref{e:reduction} for $K$ iterations, optimizing each $F_t (\w)$ to $\epsilon$-suboptimality with an algorithm $\calA$, \ie, $\bbE_{\calA} \sbr{F_t (\w_t) - F_t (\w_t^*) \,|\, \w_{t-1}} \le \epsilon$ for $t=1,\dots,K$, and pick an iteration index $R \in \cbr{1,\dots,K}$ uniformly at random, we have
    \begin{align*}
      \bbE_{R,\calA} \norm{\nabla \phi(\w_R)}^2  \le
      \frac{ 4 \gamma \rbr{\phi (\w_0) - \phi^*}}{K}  + (4 \beta + 8 \gamma) \epsilon.
    \end{align*}
  \end{cor}

\subsection{The basic MP algorithm}
\label{sec:mp-smooth-stability}

We now turn to solving the subproblems in~\eqref{e:reduction}, \ie, to minimizing the \emph{stochastic} convex objective $F_t (\w)$, to which we only have access through i.i.d. samples of the underlying distribution. 
A natural approach for this problem is through (approximate) empirical risk minimization: we draw $b$ i.i.d. samples $Z_t=\cbr{\xi_1^t, \dots, \xi_b^t}$ and compute
\begin{align} \label{e:mp-at-iteration-t}
\text{(MP)}\quad
  \w_t \approx \hw_t = \argmin_{\w}\; \hF_t (\w) \quad \text{where} \; 
  \hF_t (\w) := \frac{1}{b} \sum_{i=1}^b \ell (\w,\xi_i^t) + \frac{\gamma}{2} \norm{\w - \w_{t-1}}^2.
\end{align}
Since this procedure optimizes on a minibatch of samples the nonlinearized loss (as opposed to linearized loss which leads to minibatch SGD), we call this procedure ``minibatch-prox'' (MP). Similar algorithms were proposed previously for optimizing convex problems: a version of this algorithm with $b=1$ was studied under the names ``passive aggressive'' update~\citep{Crammer_06a} and ``implicit gradient descent''~\citep{KulisBartlet10a} for online learning, was used for optimizing finite-sum objectives~\citep{Bertsek15a,Defazio16a}, and the version of large $b$ was more recently used as the building block to develop communication-efficient distributed algorithms for stochastic convex optimization~\citep{Li_14e,Wang_17b}.

In order to ensure small suboptimality in the population objective $F_t (\w)$ as required by Corollary~\ref{cor:reduction}, we need to bound the difference between the empirical and the population objectives at $\w_t$, or in other words, the generalization performance of $\w_t$. 
In the following lemma, we provide the generalization guarantee based on the notion of stability~\citep{BousquetElisseef02a,Shalev_09a}. Our result establishes the connection between the stability of ERM and the variance of stochastic gradients, which have been two major and seemingly parallel assumptions for deriving stochastic learning guarantees, and this formal connection appears to be new in the literature.

\begin{lem} \label{lem:stability}
  Consider the stochastic optimization problem 
  \begin{align*}
    F (\w) := \phi (\w) +  r (\w) = \bbE_{\xi} \left[\ell (\w, \xi) \right]  + r (\w)
  \end{align*}
  where the instantaneous loss $\ell (\w, \xi)$ is $\sigma$-almost convex and $\beta$-smooth in $\w$, and satisfies the variance condition $\bbE_{\xi} \norm{\nabla \ell (\w,\xi) - \nabla \phi (\w) }^2 \le V^2$, and the data-independent regularizer $r (\w)$ is $\gamma$-strongly convex with $\gamma > \sigma$. Denote $\w^*=\argmin_{\w}\; F(\w)$. 

  Let $Z=\{ \xi_1,\dots,\xi_b \}$ be i.i.d. samples and 
  \begin{align*}
    \hw = \argmin_{\w} \ \hF (\w) \qquad \text{where}\quad \hF (\w) := \hphi (\w) + r (\w) = \frac{1}{b} \sum_{i=1}^b \ell (\w, \xi_i)  + r (\w).
  \end{align*}

  Assume that $(\gamma - \sigma) b \ge 2 (\sigma + \beta)$. Then the following stability results hold.
  \begin{enumerate}
  \item For the regularized empirical risk minimizer $\hw$, we have 
    \begin{align*}
      \bbE_{Z} \sbr{F (\hw) - F (\w^*)} \le 
      \bbE_{Z} \sbr{ \phi (\hw) - \hphi (\hw) } \le \frac{8 V^2}{(\gamma - \sigma) b}. 
    \end{align*}
  \item If a possibly randomized algorithm $\calA$ minimizes $\hF (\w)$ up to $\delta$-suboptimality, \ie, $\calA$ returns an approximate solution $\tw$ such that
    \begin{align*}
      \bbE_{Z,\calA} \sbr{ \hF (\tw) -  \hF (\hw) } \le \delta,
    \end{align*}
    we have
    \begin{align*}
    \bbE_{Z,\calA} \sbr{F (\tw) - F (\hw)}
    & \le  \frac{8 V^2}{(\gamma - \sigma) b} +  \frac{2 (\beta+\gamma) \delta}{\gamma - \sigma}, \\
      \bbE_{Z,\calA} \sbr{F (\tw) - F (\w^*)}
      & \le  \frac{16 V^2}{(\gamma - \sigma) b} +  \frac{2 (\beta+\gamma) \delta}{\gamma - \sigma} .
    \end{align*}
  \end{enumerate}
%%% Moreover, if $\ell(\w,\xi)$ is instead $\alpha$-strongly convex for $\alpha\ge 0$, the above inequalities continue to hold with $\sigma$ replaced by $-\alpha$.
\end{lem}

The above lemma shows that, the convergence rate of the stochastic objective $F_t (\w)$ by ERM is of the order $\calO \rbr{\frac{1}{b}}$, and thus in order to achieve $\epsilon$-suboptimality in $F_t (\w)$, it suffices to exactly solve the ERM problem defined by $\calO \rbr{\frac{1}{\epsilon}}$ samples. Moreover, the second part of Lemma~\ref{lem:stability} shows that as long as we minimize the ERM objective $\hF_t (\w)$ to suboptimality $\delta=\calO (\epsilon)$, the population suboptimality remains of the order $\calO (\epsilon)$. Allowing inexact minimization enables us to use state-of-the-art methods for convex optimization. 

\begin{rmk}[Relaxation of individual almost convexity]
\label{rmk:relaxation}
From the proof of Lemma~\ref{lem:stability}, we observe that the most important usage of the almost convexity is to ensure that the regularized empirical objective $\hF (\w)$ is $(\gamma - \sigma)$-strongly convex; as long as this holds, we obtain the $\Omega \rbr{ \frac{V^2}{(\gamma - \sigma) b} }$ stability for $b \ge \frac{4 \beta}{\gamma - \sigma}$, without almost convexity of the instantaneous losse.
As a result, we may relax our assumption to the population loss $\phi(\w)$ being $\sigma$-almost convex. Based on the $\beta$-smoothness of $\ell (\w, \xi)$ and matrix concentration, with high probability it holds that $\hphi (\w)$ is $\Omega (\sigma)$-almost convex as long as $b  \succeq \frac{\beta^2 \log d}{\sigma^2}$, where $d$ is the dimensionality of $\w$. We will see shortly that our optimal minibatch size shall increase with the final accuracy, so as $\varepsilon \rightarrow 0$, we indeed have that all subproblems are sufficiently strongly convex and consequently, the same results hold with high probability.
\end{rmk}

\subsection{Convergence of the basic MP algorithm}

We are now ready to analyze the convergence property of MP. 

\begin{thm}\label{thm:convergence} Set $\gamma = \sigma + \sqrt{\frac{32 (\beta + 2 \sigma) V^2 K}{\rbr{\phi (\w_0) - \phi^*} b}}$ in the minibatch-prox algorithm~\eqref{e:mp-at-iteration-t}. 
And assume that for each iteration $t$, we draw $b$ samples to approximate $F_t(\w)$ with $b \ge \frac{ 2 (\sigma + \beta) }{\gamma - \sigma}$, and minimize the ERM objective $\hF_t (\w)$ using a randomized algorithm $\calA$, such that $\bbE_{Z_t, \calA} \sbr{ \hF_t (\w_t) - \hF_t (\hw_t) \,|\, \w_{t-1}} \le \delta = \frac{8V^2}{(\beta+\gamma)b}$. If we pick an iteration index $R \in \cbr{1,\dots,K}$ uniformly at random, we have
\begin{align*}
  \bbE_{R,\calA} \norm{\nabla \phi(\w_R)}^2 
  \le  \frac{ 4 \sigma (\phi (\w_0) - \phi^*) }{K} +  \frac{256 V^2}{b} 
  + \frac{32 V \sqrt{ (2\beta+4\sigma) (\phi (\w_0) - \phi^*)} }{\sqrt{b K}}
\end{align*}
\end{thm}

We now add a few remarks regarding the convergence result in Theorem~\ref{thm:convergence}. First, the algorithm does not converge to an approximate critical point for very small $b$, and in fact to satisfy~\eqref{e:critical} it is necessary to have $b \succeq \frac{V^2}{\varepsilon^2}$. 
But once % $b \ge \frac{64 V^2 T}{\sigma (\phi (\w_0) - \phi^*)}$,
$b\succeq \frac{V^2 K}{\beta (\phi(\w_0)-\phi^*)}$, the second term $\frac{256 V^2}{b}$ is dominated by the other two terms, and we obtain the convergence guarantee
  \begin{align*}
  \bbE_{R,\calA} \norm{\nabla \phi(\w_R)}^2 
  \le  \calO \rbr{ \frac{  \sigma (\phi (\w_0) - \phi^*) }{K} 
    +  \frac{V \sqrt{\beta (\phi (\w_0) - \phi^*)} }{\sqrt{b K}}  }.
\end{align*}
Compare this with the convergence rate of RSAG given in~\eqref{e:rsag}. We note that, while the second term (``sample error'') is of the same order for both methods, the first term (``optimization error'') in our method depends on $\sigma$ instead of $\beta$. The first term also agrees with the number of subproblems resulted from the convexification procedure (cf. the discussion in Section~\ref{sec:reduction}).

Further assume that the second term is dominant, which is true as long as $b =\calO \rbr{ \frac{\beta V^2 K}{\sigma^2 (\phi (\w_0) - \phi^*)} }$, then to find an critical point satisfying~\eqref{e:critical}, the sample complexity is $N(\varepsilon)=\calO \rbr{ \frac{V^2 \beta (\phi (\w_0) - \phi^*)}{\varepsilon^4} }$, while the maximum minibatch size (that maintains sample efficiency) and the iteration complexity using this minibatch size are respectively
\begin{align} \label{e:mp-complexity}
  b_{\MP}=\calO \rbr{ \frac{ \sqrt{\beta} V \sqrt{N(\varepsilon)}}{\sqrt{\sigma^2 (\phi (\w_0) - \phi^*)}} } = \calO \rbr{\frac{\beta V^2}{\sigma \varepsilon^2}},
  \qquad 
  K_{\MP}=\calO \rbr{ \frac{\sigma (\phi(\w_0)-\phi^*)}{\varepsilon^2} }.
\end{align}
Increasing the minibatch size reduces the number of iterations in the
regime of $b \le b_{\MP}$.  Therefore, MP achieves the same sample
error as RSAG using the same level of samples, but when $\sigma \ll
\beta$, MP allows us to use much larger minibatch size and a smaller
number of minibatches.  The caveat here, which we will address in the
next section, is that MP requires solving an optimization problem on
each minibatch, as opposed to performing a single gradient step.

% \begin{rmk} \label{rmk:stepsize}
\paragraph{Stepsize} By the optimality condition of~\eqref{e:mp-at-iteration-t}, we have $\w_t \approx \w_{t-1} - \frac{1}{\gamma} \rbr{ \frac{1}{b} \sum_{i=1}^b \nabla \ell(\w_t, \xi_i^t)}$. This update resembles that of minibatch SGD, except that the gradient is evaluated at the ``future'' iterate. Moreover, according to Theorem~\ref{thm:convergence}, the ``stepsize'' $\frac{1}{\gamma}$ roughly varies like $\sqrt{\frac{b}{K}}$ (this approximation is more accurate for smaller $b$, in which case $\sigma\ll \gamma$), which scales with $b$ if the number of total samples $N$ is fixed, and scales with $\sqrt{b}$ if the number of iterations $K$ is fixed, consistent with the findings of~\citet{Goyal_17a} and~\citet{Hoffer_17a} respectively for minibatch SGD. 
% \end{rmk}

\begin{rmk} \label{rmk:memory-efficient}
We have shown that it suffices to approximately minimize the stochastic objective $F_t (\w)$ with $\calO \rbr{\frac{1}{\varepsilon^2}}$ samples, by approximately minimizing the empirical objective $\hF_t (\w)$. But the number of samples can be too large (as $\varepsilon \rightarrow 0$) that the memory requirement is high, since we need to store this many samples and process them multiple times. In Appendix~\ref{sec:memory-efficient}, we provide a modified algorithm to resolve this issue, which achieves the same learning guarantee with the same level of total samples, and using any (sufficiently large) minibatch size.
\end{rmk}

We demonstrate our theory and the MP algorithm on deep neural networks training in Appendix~\ref{sec:expt}.

\section{Distributed implementation of MP}
\label{sec:distributed}

In the previous Section, we presented the template algorithm MP.  But
MP requires solving an optimization problem on a mini-batch at each
iteration.  We now instantiate the algorithm by suggesting specific
distributed procedures for solving this minibatch optimization
problem.  We consider two possible solvers for the convex subproblems
$\hF_t (\w)$, $t=1,\dots,K$ in MP, and discuss the resulting overall
parallel runtime and the total computational cost.

\begin{table}[t]
\centering
\caption{Comparisons between minibatch SGD and MP equipped with two convex optimizers, in terms of both the runtime and the total number of vector operations to find an approximate critical point. 
The runtime is measured by the number of batch gradient evaluations, each of which cost $\tau_b$ on a minibatch of $b$ samples, and the number of serial gradient descent updates, each of which cost $\tau_1$ on a single sample. 
Denote by $\Delta:=\phi(\w_0) - \phi^*$ the initial suboptimality. 
The (sufficient) sample complexity is denoted by $N (\varepsilon) = \calO \rbr{\frac{V^2 \beta \Delta}{\varepsilon^4}}$. We hide poly-logarithmic dependence on $(\sigma, \beta, V, \Delta, \varepsilon, b)$.  Regimes that are not sample-efficient are shadowed.}
\label{t:comparison}
\begin{tabular}{@{}|c|c|c|c|@{}}
\hline
& Minibatch size & Parallel runtime & \# vector operations \\
\hline
\hline

\multirow{3}{*}{RSAG} & 
$b\prec \frac{V^2}{\varepsilon^2}$ &
$\frac{N (\varepsilon)}{b} \times \tau_b$ & 
$N (\varepsilon)$ \\[.5ex]
\cline{2-4} &
$b \asymp \frac{V^2}{\varepsilon^2}$ &
$\frac{\beta \Delta}{\varepsilon^2} \times \tau_b$ & 
$N (\varepsilon)$  \\[.5ex]
% \cline{2-4} &
% \cellcolor{light-gray} $\frac{V^2}{\varepsilon^2} \prec b \prec \frac{\beta V^2}{\sigma \varepsilon^2}$ &
% \cellcolor{light-gray} $\frac{\beta \Delta}{\varepsilon^2} \times \tau_b$ & 
% \cellcolor{light-gray} $\frac{b \varepsilon^2}{V^2} \cdot N (\varepsilon)$  \\[.5ex]
% \cline{2-4} &
% \cellcolor{light-gray} $b \asymp \frac{\beta V^2}{\sigma \varepsilon^2}$ &
% \cellcolor{light-gray} $\frac{\beta \Delta}{\varepsilon^2} \times \tau_b$ & 
% \cellcolor{light-gray} $\frac{\beta}{\sigma} \cdot N (\varepsilon)$  \\[.5ex]
% \cline{2-4} &
% \cellcolor{light-gray} $b \succ \frac{\beta V^2}{\sigma \varepsilon^2}$ &
% \cellcolor{light-gray} $\frac{\beta \Delta}{\varepsilon^2} \times \tau_b$ & 
% \cellcolor{light-gray} $\frac{b \varepsilon^2}{V^2} \cdot N (\varepsilon)$  \\[.5ex]
 \cline{2-4} &
 \cellcolor{light-gray} $b \succ \frac{V^2}{\varepsilon^2}$ &
 \cellcolor{light-gray} $\frac{\beta \Delta}{\varepsilon^2} \times \tau_b$ & 
 \cellcolor{light-gray} $\frac{b \varepsilon^2}{V^2} \cdot N (\varepsilon)$  \\[.5ex]
\hline

\multirow{3}{*}{\caja{c}{c}{MP\\+ AGD}} &
% $b \prec \frac{V^2}{\varepsilon^2}$ & N/A & N/A \\ 
% \cline{2-4} & 
% $ b \asymp \frac{V^2}{\varepsilon^2}$ &
% $\frac{\beta \Delta}{\varepsilon^2} \times \tau_b$ & 
% $N (\varepsilon)$ \\[.5ex]
% \cline{2-4} & 
$\frac{V^2}{\varepsilon^2} \preceq b \prec \frac{\beta V^2}{\sigma \varepsilon^2}$ &
$\sqrt{\frac{\varepsilon^2}{V^2 b}} \cdot N (\varepsilon) \times \tau_b$ & 
$\sqrt{\frac{b \varepsilon^2}{V^2}} \cdot N (\varepsilon)$ \\[.5ex]
\cline{2-4} & 
$b \asymp \frac{\beta V^2}{\sigma \varepsilon^2}$  &
$\frac{\sqrt{\sigma \beta} \Delta}{\varepsilon^2} \times \tau_b$ & 
$\sqrt{\frac{\beta}{\sigma}} \cdot N (\varepsilon)$ \\[.5ex]
\cline{2-4} & 
\cellcolor{light-gray} $b \succ \frac{\beta V^2}{\sigma \varepsilon^2}$  &
\cellcolor{light-gray} $\rbr{\frac{b \varepsilon^2}{V^2}}^{\frac{1}{4}} \cdot \frac{\sigma^{\frac{3}{4}} \beta^\frac{1}{4} \Delta}{\varepsilon^2} \times \tau_b$ & 
\cellcolor{light-gray} $\rbr{\frac{b \varepsilon^2}{V^2}}^{\frac{5}{4}} \rbr{\frac{\sigma}{\beta}}^{\frac{3}{4}} \cdot N (\varepsilon)$ \\[.5ex]
\hline

\multirow{3}{*}{\caja{c}{c}{MP\\+ SVRG}} & 
% $b \prec \frac{V^2}{\varepsilon^2}$ & N/A & N/A \\ 
% \cline{2-4} & 
% $b \asymp \frac{V^2}{\varepsilon^2} $ &
% $ \frac{\beta \Delta}{\varepsilon^2} \times \tau_b +  \frac{\beta \Delta}{\varepsilon^2} \times \tau_1$ & 
% $N (\varepsilon)$ \\[.5ex]
% \cline{2-4} & 
$\frac{V^2}{\varepsilon^2} \preceq b \prec \frac{\beta V^2}{\sigma \varepsilon^2}$ &
$ \frac{N (\varepsilon)}{b} \times \tau_b +  \frac{\beta \Delta}{\varepsilon^2} \times \tau_1$ & 
$N (\varepsilon)$ \\[.5ex]
\cline{2-4} & 
$b \asymp \frac{\beta V^2}{\sigma \varepsilon^2}$  &
$\frac{\sigma \Delta}{\varepsilon^2} \times \tau_b + \frac{\beta \Delta}{\varepsilon^2} \times \tau_1$ & 
$N (\varepsilon)$ \\[.5ex]
\cline{2-4} & 
\cellcolor{light-gray} $b \succ \frac{\beta V^2}{\sigma \varepsilon^2}$  &
\cellcolor{light-gray} $\frac{\sigma \Delta}{\varepsilon^2} \times \tau_b + \sqrt{\frac{b \varepsilon^2}{V^2}} \cdot \frac{\sqrt{\sigma \beta} \Delta}{\varepsilon^2} \times \tau_1$ & 
\cellcolor{light-gray} $\frac{b \varepsilon^2}{V^2} \cdot \frac{\sigma}{\beta} \cdot N (\varepsilon)$ \\[.5ex]
\hline
\end{tabular}
\end{table}

\paragraph{Accelerated gradient descent}
The first choice is the (distributed) accelerated gradient descent (AGD,~\citealp{Nester04a}), in which case MP uses the same minibatch gradients as minibatch SGD does. Observe that each $\hF_t(\w)$ is both $(\beta+\gamma)$-smooth and $(\gamma - \sigma)$-strongly convex, and by our choice of $\gamma$, its condition number is $\kappa=\frac{\beta+\gamma}{\gamma - \sigma}=\calO \rbr{\sqrt{\frac{\beta (\phi(\w_0) - \phi^*) b}{V^2 K}}}$ which increases with $b$. 
By the convergence rate of AGD, when minimizing $\hF_t(\w)$, the number of gradient descent updates needed to achieve $\delta$ suboptimality is $\calO \rbr{\sqrt{\kappa} \log \frac{1}{\delta}}$.
Consequently, the total number of gradient descent updates throughout the MP algorithm is\footnote{We use the $\tilde{\calO} (\cdot)$ notation to hide poly-logarithmic dependence on $(\sigma, \beta, V, \phi (\w_0) - \phi^*, \varepsilon, b)$.} $\tilde{\calO} \rbr{ K \cdot \sqrt{\kappa} }$ % = \tilde{\calO} \rbr{ b^{-\frac{1}{2}} N(\varepsilon)^{\frac{3}{4}} }$ which decreases with $b$. 
and the total number of vector operations by the algorithm is $\tilde{\calO} \rbr{ b \cdot K \cdot \sqrt{\kappa} }$. 
We provide the total number of gradient steps and the corresponding computational cost, as functions of the problem parameters, for different regimes of minibatch size in Table~\ref{t:comparison}. Most notably, when using the maximum minibatch size $b_{\MP}$, the total number of gradient descent updates is
\begin{align} \label{e:gradient-complexity}
T_{\MPAGD} = 
\tilde{\calO} \rbr{ \frac{\sqrt{\sigma \beta} (\phi (\w_0) - \phi^*)}{\varepsilon^2} },
\end{align}
which is asymptotically smaller than $T_{\RSAG}$. In the \emph{worst
  scenario} where $\sigma=\beta$, we perform roughly the same number
of gradient descent updates as RSAG.  But when $\sigma \ll \beta$, we
can significantly reduce the parallel runtime.

\paragraph{Distributed SVRG}
Another option for optimizing $\hF_t (\w)$ is the distributed SVRG
algorithm for finite-sum problems~\citep{Lee_16b,Shamir16c}.\footnote{Other parallel optimization
  frameworks, such as DANE~\citep{Shamir_14a} and AIDE~\citep{Reddi_16b}, can be applied and analyzed similarly.} 
This algorithm alternates over two types of operations: the evaluation of
batch gradient on $b$ samples, which can distributed into multiple
machines, and many serial stochastic gradient descent steps, each of
which uses gradient computed on a single sample on one local machine. 

Denote the time cost for the two types of access by $\tau_b$ and $\tau_1$ respectively. 
According to the convergence of distributed SVRG, the runtime needed to optimize each subproblem to sufficient accuracy is 
\begin{align*}
\tilde{\calO} \rbr{ \tau_b + \kappa \tau_1}.
\end{align*}
Similar to the case of AGD, we provide the runtime and the total
computational cost by MP+SVRG in Table~\ref{t:comparison}.

The parallel runtime of the two types of operations, calculating the
gradient of a mini-batch of size $b$ in parallel across machines, and
calculating the gradient on a single sample locally, are not directly
relate-able.  From a pure parallel computation perspective, one could
argue that $\tau_b$ does not depend much on $b$, since we can
distribute the computation across $b$ machines, and so $\tau_1 \approx
\tau_b$.  If this is the case, MP+SVRG does not provide any advantage
over RSAG, since MP+SVRG's parallel runtime will be dominates by
$\frac{\beta \Delta}{\epsilon^2}$, the same as RSAG's.  

However, more realistically, even in a parallel setting, we would
expect computing the gradient of a single point on a single machine,
without any communication, would be much quicker than parallel
computation of a minibatch.  If indeed $\tau_1 \ll \tau_b$, and in
particular 
\begin{equation}\label{e:tau1}
\tau_1 = \calO \rbr{\frac{\sigma}{\beta} \tau_b},
\end{equation} 
MP+SVRG dominates both MP+AGD {\em and} RSAG: the $\tau_b$ term is then the
dominant term in its parallel runtime, with an improvement over both
MP+SVRG, yet the total computational cost (number of individual
gradient computations) does not increase over the optimal cost of
RSAG, and is thus much smaller than MP+AGD---we get a reduction in
parallel runtime without any additional computational cost.  In
particular, with a minibatch of size $b_{\MP}$, and under the assumption~\eqref{e:tau1}, we get a parallel runtime of
\begin{equation}
  \label{e:TMPSVRG}
T_{\MPSVRG} = 
\tilde{\calO} \rbr{ \frac{\sigma(\phi (\w_0) - \phi^*)}{\varepsilon^2} },
\end{equation}
with optimal computational cost $N(\varepsilon)$.  The requirement
\eqref{e:tau1} is very reasonable, especially for large $b$,
considering the required ratio does not depend on $b$.  All we require
is that parallel computation of a gradient on a large minibatch is at
least a constant factor more expensive than an individual gradient
computation on one point.

\section{Discussion}
\label{sec:discussion}

In this work, we have focused on stochastic nonconvex optimization using only noisy first-order gradient information, and made a step toward large minibatch training. Our results suggest that it is beneficial to perform better optimization on each minibatch than a single gradient descent, when the minibatch size is too large to be sample-inefficient in minibatch SGD.

Unfortunately, we could not yet remove the ``optimization error'' altogether from the convergence rate as is achievable for stochastic convex optimization using minibatch-prox~\citep{Wang_17b}: in comparison to the convergence rate~\eqref{e:acc-minibatch-sgd-convex} by accelerated minibatch SGD, the minibatch-prox algorithm on a convex $\phi(\w)$ provides the guarantee 
$  \bbE \sbr{ \phi (\w) - \phi (\w_*) } \le \calO \rbr{  \frac{\norm{\w_0 - \w_*}}{\sqrt{b K}} } $, so that one could use any minibatch size $b$ while maintaining sample efficiency. 
Nor could we significantly reduce it as accelerated minibatch SGD achieved, again in the convex case (compare the first term in~\eqref{e:acc-minibatch-sgd-convex} and~\eqref{e:rsag}). We also do not know if the ``sample error'' is the statistical limit for the class of problems considered here, and if a refined analysis of minibatch SGD (that makes use of the $\sigma$-almost convexity) can show the same convergence rate, which would tell if our more complicated algorithms are indeed necessary. 
In combination with the convexification procedure,~\citet{Carmon_17a} additionally made use of curvature information in the Hessian (which can be efficiently obtained for deep learning models,~\citealp{Pearlm94a,Marten10a}) so as to further reduce the ``optimization error'': for their ERM algorithm, the number of gradient steps needed is $\calO \rbr{\frac{1}{\varepsilon^{7/4}}}$, rather than $\calO \rbr{\frac{1}{\varepsilon^2}}$ for minibatch SGD and our algorithms. We suspect that using the same technique in stochastic optimization may yield similar improvement, at the cost of a more complex algorithm. 

% We think the use of the convexification procedure for stochastic nonconvex optimization is elegant, as it enables the application of existing techniques for the generalization performance of stochastic convex optimization (in particular, we used the stability of ERM in this work), which are much better understood than for nonconvex learning. It is possible that better concentration results that are more tailored to specific problems (such as those in~\citealp{Zhang_17c}) translate into overall faster rates.

%% Future: high probability, constrained case

%% Dynamically choose $\gamma$ based on $V_t$?

%% Are there results showing that for DNN objectives, sigma is much less than beta?

\section*{Acknowledgement}
Weiran Wang would like to thank Michael Maire for inspiring discussions on~\citet{Goyal_17a}, and Zeyuan Allen-Zhu for helpful discussions on the convexification procedure.

\appendix
\section{Experiments}
\label{sec:expt}

We now demonstrate our theory and the basic MP algorithm with an illustrative example. 
We train a neural network with $2$ tanh hidden layers of $512$ units each, and a softmax output layer to perform $10$-way digit classification on the infinite MNIST dataset~\citep{Loosli_07a}. 
The dataset is randomly split into $8\times 10^6$ samples for training and $10^5$ samples for testing. To mimic the stochastic setting, we allow each method to load the training set into memory only once (this is equivalent to a single training epoch for minibatch SGD).

\paragraph{Performance of minibatch SGD}
We carefully tune the training hyperparameters by grid search for minibatch SGD with momentum, which remains a very strong method for training deep models in practice: the fixed learning rate is selected from $\cbr{0.001,\, 0.01,\, 0.05,\, 0.1,\, 0.5}$, and the momentum parameter from $\cbr{0,\, 0.5,\, 0.8,\, 0.9,\, 0.99,\, 0.995}$.

We vary the minibatch size $b$ in $\cbr{200,\, 1000,\, 2000,\, 10000}$, and for each each $b$ select the optimal combination of learning rate and momentum based on the objective on the test set. The test set objective vs. number of samples processed for different minibatch sizes are given in Figure~\ref{f:SGD} (left plot). We note that this type of learning curve (or error vs. epoch) is typically used for evaluating learning methods (\eg, ~\citealp{Goyal_17a}), and is quite reasonable since the number of samples processed corresponds to the total energy spent. Observe that with smaller $b$, minibatch SGD converges to lower objective function values (although we have trained the neural network for only one epoch, the trained model with $b=200$ has a cross-entropy loss of $0.0026$ and a low classification error rate of $0.062\%$ on the test set). 
The difference in final test objectives is small for $b=200$ and $b=1000$,
%% (thus, with perfect parallelization, $5\times$ speedup is achieved without loss of accuracy), 
but we start to see clear degradation of accuracy for $b=10000$.

\begin{figure}[t]
\centering
\psfrag{obj}[b][]{test objective}
\psfrag{samples}[][]{\# fresh samples}
\psfrag{updates}[][]{\# updates}
\psfrag{SGD, b=200}[l][l][0.75]{SGD, $b=200$}
\psfrag{SGD, b=1000}[l][l][0.75]{SGD, $b=1000$}
\psfrag{SGD, b=2000}[l][l][0.75]{SGD, $b=2000$}
\psfrag{SGD, b=4000}[l][l][0.75]{SGD, $b=4000$}
\psfrag{SGD, b=10000}[l][l][0.75]{SGD, $b=10000$}
\begin{tabular}{@{}cc@{}}
\includegraphics[width=0.490\linewidth]{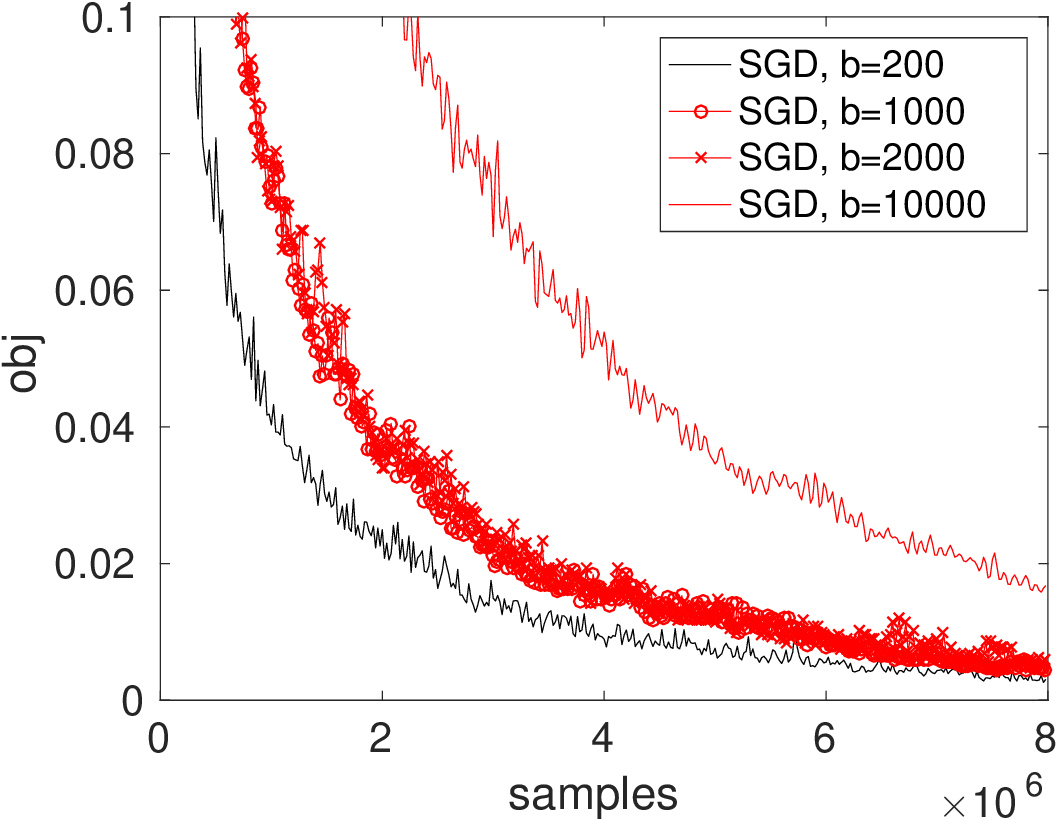} & 
\raisebox{0.01\linewidth}{
\includegraphics[width=0.500\linewidth]{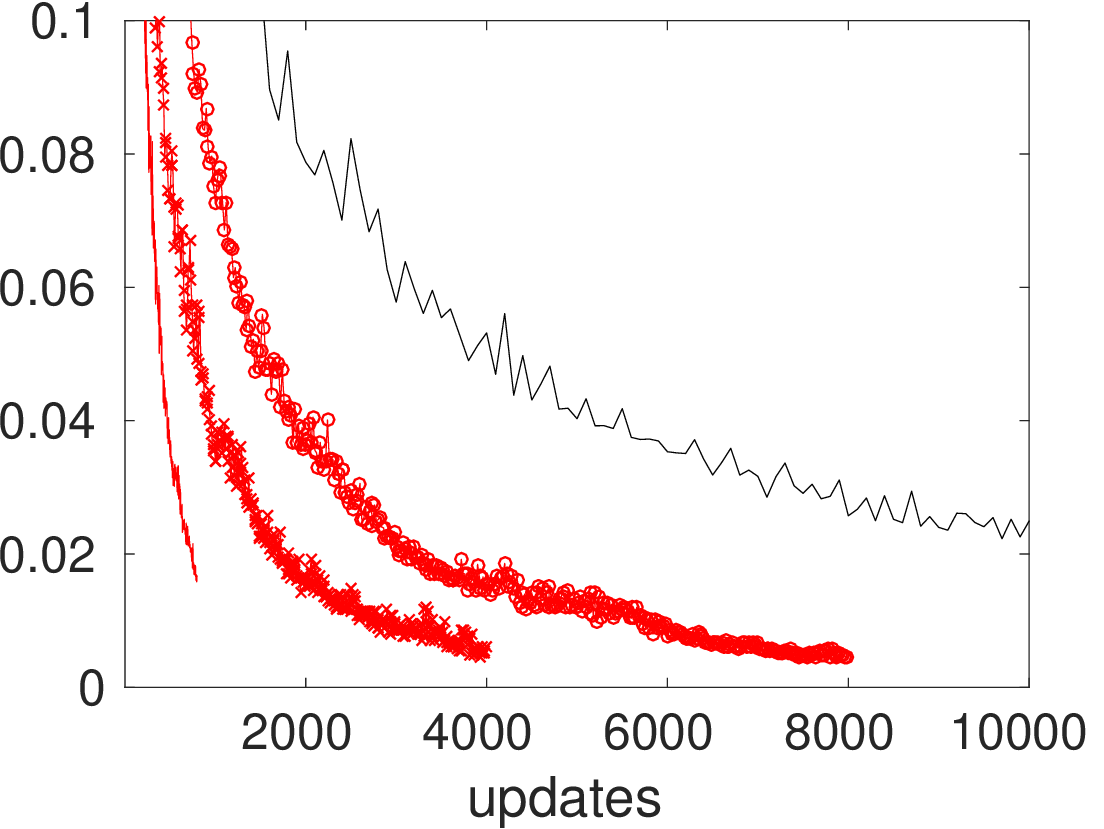} }
\end{tabular}
\vspace*{-2ex}
\caption{Performance of minibatch SGD with different minibatch sizes.}
\label{f:SGD}
\end{figure}

On the other hand, we provide test set objective vs. number of updates in Figure~\ref{f:SGD} (right plot). And we observe that the decrease of objective is much steeper for larger $b$, implying that a single gradient descent update with large $b$ is of higher quality. 

\paragraph{Performance of MP} We now show that MP can achieve
significantly higher accuracy with $b=10000$, even slight improving over that of minibatch SGD with $b=200$, using the same number of fresh samples and moderate number of gradient updates. 

In our MP implementation, we approximately solve the subproblems on each minibatch with $g$ gradient descent steps with momentum\footnote{Gradient descent with momentum and accelerated gradient descent have similar forms of updates.}, yielding a training procedure similar to that of minibatch SGD, except that each large minibatch is kept in memory for $g$ steps before switching to the next one, and that the gradient contains a \emph{retraction} term $\gamma (\w - \w_{t-1})$ from the quadratic regularization. The learning rate and momentum parameter are tuned over a smaller grid around the optimal values for minibatch SGD. 
We tune $g$ over $\cbr{5,\, 10,\, 20,\, 50}$ and the regularization
parameter $\gamma$ over the grid $\cbr{0,\, 10^{-6},\, 10^{-4},\,
  10^{-2},\, 1}$. 
This implementation reduces to minibatch SGD when $\gamma=0$ and $g=1$. 

For each $\gamma$, we select the combination of rest hyperparameters that gives the lowest test objective. Learning curves (objective vs. \# fresh samples, and objective vs. \# updates) for different values of $\gamma$ are given in Figure~\ref{f:MP}, where we also compare with the learning curves of minibatch SGD at $b=200$ and $b=10000$.
Observe that, with moderate values of $g$, MP can match the objective vs. \# fresh samples curve of minibatch SGD at $b=200$ so that it is sample-efficient (the trained model with $\gamma=10^{-2}$ and $g=50$ has a cross-entropy loss of $0.0004$ and a classification error rate of $0.012\%$). 
On the other hand, MP is close to minibatch SGD at $b=10000$ for the
objective vs. \# updates learning curve, so that each step is still of high quality and quickly decreases the objective. 

\begin{figure}[t]
\centering
\psfrag{obj}[b][]{test objective}
\psfrag{samples}[][]{\# fresh samples}
\psfrag{updates}[][]{\# updates}
\psfrag{SGD, b=200}[l][l][0.62]{SGD, $b=200$}
\psfrag{SGD, b=10000}[l][l][0.62]{SGD, $b=10^4$}
\psfrag{MP, ga=0, gd=50}[l][l][0.62]{MP, $b=10^4$, $\gamma=0$, $\hspace*{1.5em} g=50$}
\psfrag{MP, ga=1e-06, gd=50}[l][l][0.62]{MP, $b=10^4$, $\gamma=10^{-6}$, $g=50$}
\psfrag{MP, ga=0.0001, gd=50}[l][l][0.62]{MP, $b=10^4$, $\gamma=10^{-4}$, $g=50$}
\psfrag{MP, ga=0.01, gd=50}[l][l][0.62]{MP, $b=10^4$, $\gamma=10^{-2}$, $g=50$}
\psfrag{MP, ga=1, gd=5}[l][l][0.62]{MP, $b=10^4$, $\gamma=1$, $g=5$}
\begin{tabular}{@{}cc@{}}
\includegraphics[width=0.490\linewidth]{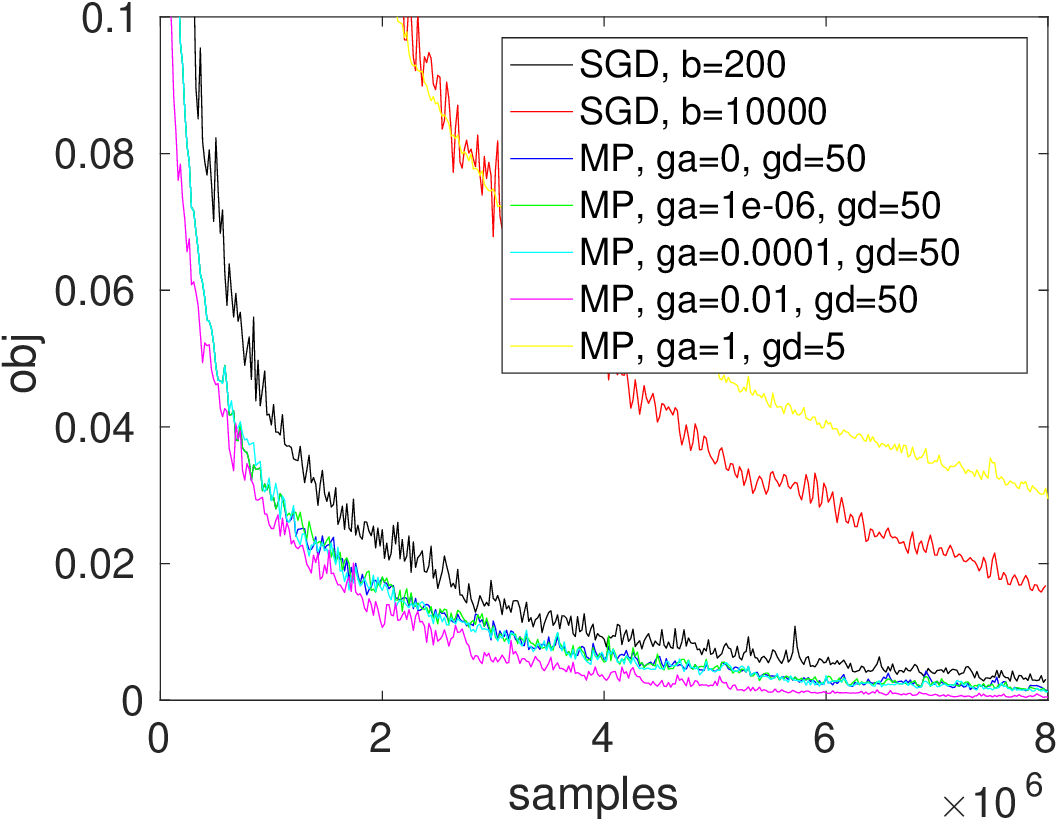} & 
\raisebox{0.01\linewidth}{
\includegraphics[width=0.500\linewidth]{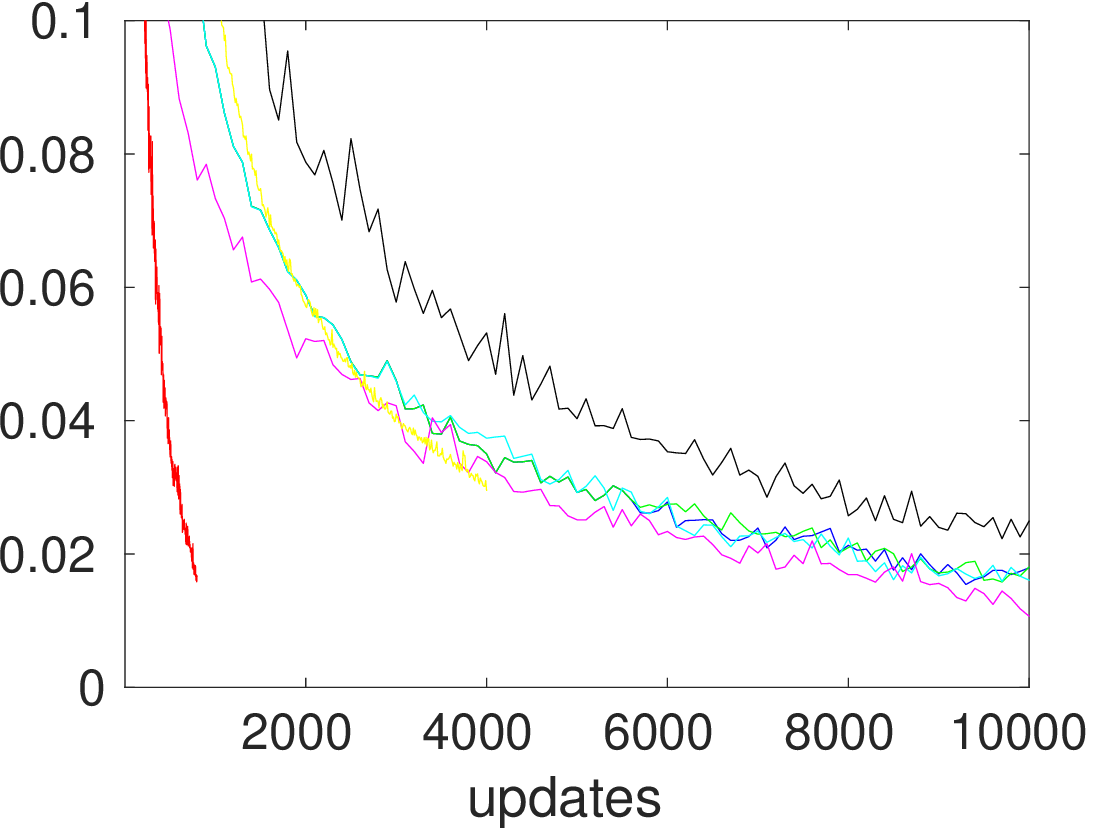} }
\end{tabular}
\vspace*{-2ex}
\caption{Performance of MP+AGD with different $\gamma$.}
\label{f:MP}
\end{figure}

We have seen in Figure~\ref{f:MP} that in fact $\gamma=0$ works quite well without the retraction term (difference in final objectives are not significant for small $\gamma$), implying that simply processing the same large minibatch multiple times in minibatch SGD helps improve the sample efficiency. For this simple method, we provide the learning curves at different $g$ values in Figure~\ref{f:NUMAGD}. 
From this figure, it is clear that we have speedup in terms of \# gradient
updates (or parallel runtime) at various levels of test objective (and
hence classification error rate). For example, to obtain a test objective
of $0.01$ (roughly corresponding to an acceptable error rate of $0.32\%$),
we can use MP with $b=10000$ and $g=5$ for about $2600$ updates, while
minibatch SGD with $b=200$ obtains the same objective after about $17000$
updates. 
This result demonstrates the success of a practical version of our
method: if we have a constant stream of data, we can perform several
gradient steps on each large minibatch in a parallel system to improve
runtime, without losing much statistical precision.

\begin{figure}[t]
\centering
\psfrag{obj}[b][]{test objective}
\psfrag{samples}[][]{\# fresh samples}
\psfrag{updates}[][]{\# updates}
\psfrag{SGD, b=200}[l][l][0.65]{SGD, $b=200$}
\psfrag{SGD, b=10000}[l][l][0.65]{SGD, $b=10^4$, $\gamma=0$, $g=1$}
\psfrag{MP, bb, gm=0, gd=5}[l][l][0.65]{MP, $b=10^4$, $\gamma=0$, $g=5$}
\psfrag{MP, bb, gm=0, gd=10}[l][l][0.65]{MP, $b=10^4$, $\gamma=0$, $g=10$}
\psfrag{MP, bb, gm=0, gd=20}[l][l][0.65]{MP, $b=10^4$, $\gamma=0$, $g=20$}
\psfrag{MP, bb, gm=0, gd=50}[l][l][0.65]{MP, $b=10^4$, $\gamma=0$, $g=50$}
\begin{tabular}{@{}cc@{}}
\includegraphics[width=0.490\linewidth]{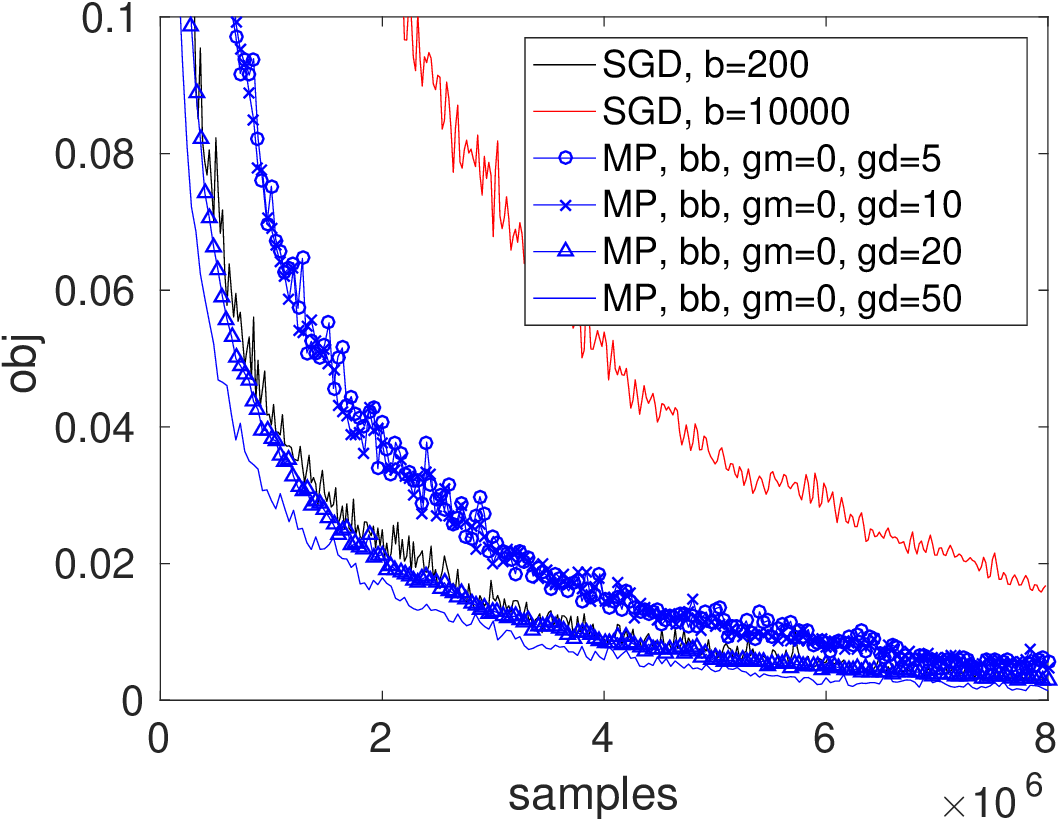} & 
\raisebox{0.01\linewidth}{
\includegraphics[width=0.500\linewidth]{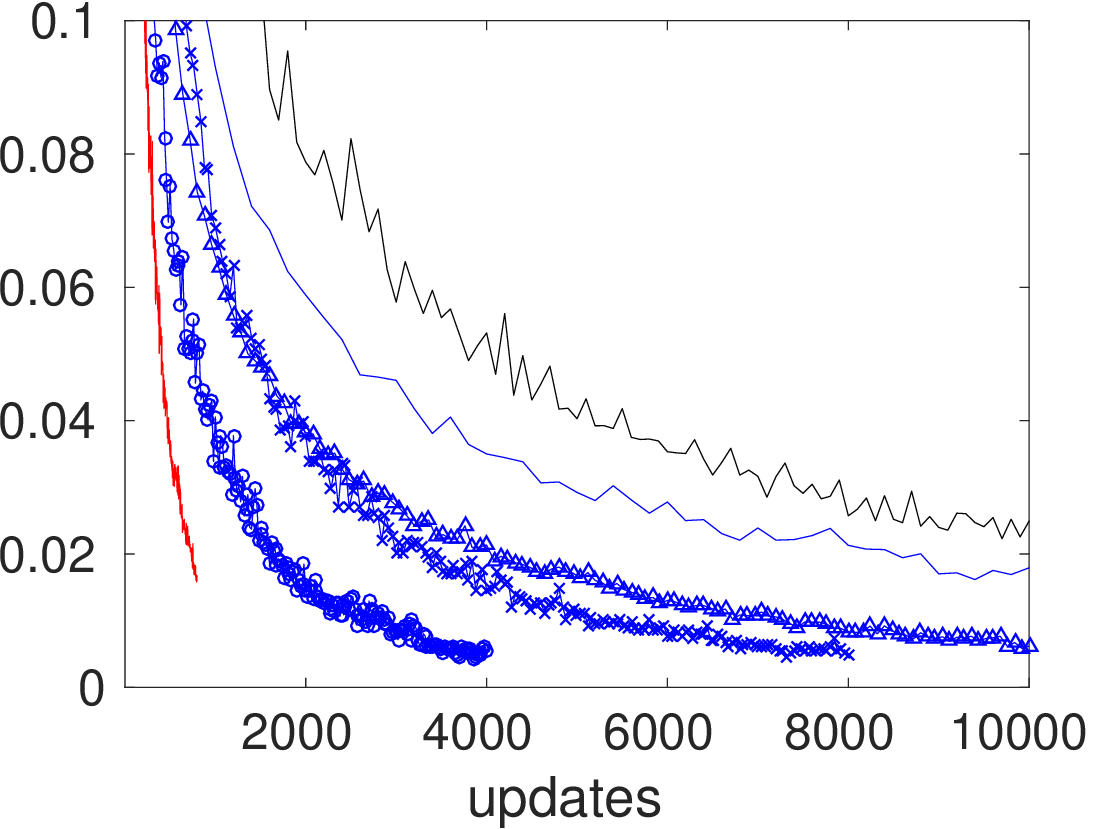} }
\end{tabular}
\vspace*{-2ex}
\caption{Performance of MP+AGD with $\gamma=0$ and different number of gradient updates per minibatch $g$.}
\label{f:NUMAGD}
\end{figure}

\paragraph{Choice of hyperparameters}
We comment on how to select in practice the hyperparameters---minibatch
size $b$, regularization parameter $\gamma$, and the number of gradient
steps on each minibach $g$. Intuitively,
one may set $b$ to be as large as possible to fully utilize the parallel
system, and set the number of steps $g$ to be relatively small to use more fresh
samples. Finally, $\gamma$ is a type of stepsize and is better tuned on a
validation set, as one would do for minibatch SGD. These are essentially
the principles we followed in the above experiments.

\section{A memory-efficient version of MP}
\label{sec:memory-efficient}

We have shown in previous sections that it suffices to approximately minimize the stochastic objective $F_t (\w)$ with $\calO \rbr{\frac{1}{\varepsilon^2}}$ samples, by approximately minimizing the empirical objective $\hF_t (\w)$. But the number of samples can be too large (as $\varepsilon \rightarrow 0$) that the memory requirement is high, since we need to store this many samples and process them multiple times. In this section, we provide a modified algorithm to resolve this issue, which achieves the same learning guarantee with the same level of total samples, and using any (sufficiently large) minibatch size.

The modified algorithm is based on the analysis of minibatch-prox by~\citet{Wang_17b} for convex objectives. The authors showed that for Lipschitz and strongly convex stochastic objective, the minibatch-prox algorithm achieves the optimal $\calO (1/n)$ rate\footnote{This rate is optimal in the sense of~\citet{NemirovYudin83a} and~\citet{Agarwal_12a}.} using $n$ total samples and any minibatch size. We can therefore apply their results to the problem of $\min_{\w} F_t (\w)$ at each iterations. 

In this section, we use $F (\x) = \phi (\x) + \frac{\gamma}{2} \norm{\x - \y}^2$ to denote the stochastic objective $F_t (\w)$ at any iteration, which is $(\beta+\gamma)$-smooth and $(\gamma - \sigma)$-strongly convex in $\x$. 
The lemma below is parallel to~\citet[Theorem~8]{Wang_17b}. Its proof is also similar to theirs, with the difference being the stability used: theirs used stability for Lipschitz losses, whereas ours use the stability for smooth loss given in Lemma~\ref{lem:stability}. 

\begin{lem} \label{lem:small-minibatch}
  Assume the same conditions of Lemma~\ref{lem:stability} on the instantaneous loss. 
  Consider the following iterative procedure: for $s=1,\dots,S$
  \begin{gather*}
    \x_s \approx \hx_s = \argmin_{\x}\; \hG_s (\x) \\
 \text{where} \; 
    \hG_s (\x) := \hF(\x) + \frac{\rho_s}{2} \norm{\x - \x_{s-1}}^2
    = \frac{1}{m} \sum_{i=1}^{m} \ell (\x,\xi_i^s) + \frac{\gamma}{2} \norm{\x - \y}^2 + \frac{\rho_s}{2} \norm{\x - \x_{s-1}}^2, 
  \end{gather*}
  where $\rho_s>0$, and $Z_s=\cbr{\xi_1^s, \dots, \xi_m^s}$ are $m$ i.i.d. samples drawn from the underlying distribution at iteration $s$.
  Let $\x_s$ be the output of a randomized algorithm $\calA$ satisfying $\bbE_{Z_s, \calA} \sbr{\hG_s (\x_s) - \hG_s (\hx_s)} \le \eta_s$. Then with the following choices of parameters:
  \begin{align*}
    m \ge \frac{2 (\sigma + \beta)}{\gamma - \sigma}, \qquad
    \rho_s = \frac{(\gamma - \sigma) (s-1)}{2}, \qquad
    \eta_s = \frac{V^2 S}{(\beta+\gamma) m} \cdot \frac{1}{s^5},
  \end{align*}
  we have for $\bar{\x}_S = \frac{2}{S (S+1)} \sum_{s=1}^S s \x_s$ that
  \begin{align*}
    \bbE \sbr{F (\bar{\x}_S) - F (\x^*)}  \le  \frac{200 V^2}{(\gamma - \sigma) m S}.
  \end{align*}
\end{lem}

Let the total number of samples used in this procedure be $b=mS$. This lemma shows that the $\calO \rbr{\frac{1}{b}}$ convergence rate for $F(\x)$ (as in Lemma~\ref{lem:stability}) is still achievable, by iteratively drawing smaller minibatches and solving one simpler ERM on each.

This approach leads to an algorithm with intuitively two levels of minibatch-prox, one for the convexification of nonconvex objective, and the other for memory efficiency. We provide the sketch of the resulting algorithm in Algorithm~\ref{alg:mp2}.

\begin{algorithm}[t]
  \caption{Memory-efficient minibatch-prox for  stochastic nonconvex optimization $\min_{\w}\; \phi(\w)$.}
  \label{alg:mp2}
  \renewcommand{\algorithmicrequire}{\textbf{Input:}}
  \renewcommand{\algorithmicensure}{\textbf{Output:}}
  \begin{algorithmic}
    \STATE Initialize $\w_0$.
    \FOR{$t=1,2,\dots,K$}
    \STATE \COMMENT{Approximately solve $\min_{\w}\; \phi(\w) + \frac{\gamma}{2} \norm{\w - \w_{t-1}}^2$}
    \STATE Intialize inner loop $\x_0^{(t)}$
      \FOR{$s=1,2,\dots,S$}
    \STATE Draw $m$ fresh samples $Z_s^{(t)}=\cbr{\xi_1^{t,s},\dots,\xi_m^{t,s}}$
    \STATE Approximately compute 
    \begin{align*}
      \x_s^{(t)} \leftarrow \min_{\x}\; \frac{1}{m} \sum_{i=1}^m \ell(\x, \xi_i^{t,s}) + \frac{\gamma}{2} \norm{\x - \w_{t-1}}^2 + \frac{\rho_s}{2} \norm{\x - \x_{s-1}^{(t)}}^2
      \end{align*}
      by accelerated gradient descent or finite-sum methods
      \ENDFOR
      \STATE $\w_t \leftarrow \frac{2}{S (S+1)} \sum_{s=1}^S s \x_s^{(t)}$
      \ENDFOR
      \ENSURE Pick $R \in \cbr{1,\dots,K}$ uniformly at random and return $\w_R$.
    \end{algorithmic}
  \end{algorithm}

\section{Proof of Lemma~\ref{lem:stopping}}

  \begin{proof}
    Due to the $(\beta + \gamma)$-smoothness of $F_t (\w)$ and the optimality condition that $\nabla F_t (\w_t^*)=\0$, we have~\citep[Theorem~2.1.5]{Nester04a}
    \begin{align*}
      \bbE_{\calA} \norm{\nabla F_t (\w_t)}^2 
      \le 2 (\beta + \gamma) \cdot \bbE_{\calA} \sbr{ F_t (\w_t) - F_t (\w_t^*) }
      \le 2 (\beta + \gamma) \epsilon.
    \end{align*}
    Then, by the definition of $F_t (\w)$, it holds that 
    \begin{align*}
      \bbE_{\calA} \norm{ \nabla \phi (\w_t) }^2 
      & = \bbE_{\calA} \norm{ \nabla F_t (\w_t) - \gamma (\w_t - \w_{t-1}) }^2 \\
      & \le 2 \bbE_{\calA} \norm{ \nabla F_t (\w_t) }^2 + 2 \gamma^2 \bbE_{\calA} \norm{ \w_{t-1} - \w_t }^2  \\
      & \le 4 (\beta + \gamma) \epsilon + 2 \gamma^2 \bbE_{\calA} \norm{ \w_{t-1} - \w_t }^2 
    \end{align*}
    where we have used the fact that $(x + y)^2 \le 2 (x^2 + y^2)$ in the first inequality.
  \end{proof}

\section{Proof of Corollary~\ref{cor:reduction}}

  \begin{proof}
    By the definition of $F_t (\w)$, we have 
    \begin{align*}
      \bbE_{\calA} \sbr{ \phi (\w_{t-1}) - \phi (\w_t) } 
      & = \bbE_{\calA} \sbr{ F_t (\w_{t-1}) - F_t (\w_t) + \frac{\gamma}{2} \norm{\w_{t-1} - \w_t}^2 } \\
      & = F_t (\w_{t-1}) - F_t (\w_t^*) + \bbE_{\calA} \sbr{F_t (\w_t^*) -  F_t (\w_t) + \frac{\gamma}{2} \norm{\w_{t-1} - \w_t}^2 } \\
      & \ge \frac{\gamma}{2} \bbE_{\calA} \norm{\w_{t-1} - \w_t}^2 - \epsilon
    \end{align*}
    where the expectation is taken over randomness at iteration $t$. 

    Averaging this inequality over $t=1,\dots,K$ and taking expectation over randomness in all iterations, we have
    \begin{align*}
      \frac{1}{K} \sum_{t=1}^K \bbE_{\calA} \norm{\w_{t-1} - \w_t}^2 
      \le \frac{2 \bbE_{\calA} \sbr{ \phi (\w_0) - \phi (\w_K) }}{\gamma K}  +  \frac{2 \epsilon}{\gamma}
      \le \frac{ 2 (\phi (\w_0) - \phi^*) }{\gamma K}  +  \frac{2 \epsilon}{\gamma}.
    \end{align*}

    This implies that if we randomly pick $R \in \cbr{1,\dots,K}$, it holds that 
    \begin{align*}
      \bbE_{R,\calA} \norm{\w_{R-1} - \w_R}^2 
      \le  \frac{ 2 (\phi (\w_0) - \phi^*) }{\gamma K}  +  \frac{2 \epsilon}{\gamma}.
    \end{align*}
    Plugging this into Lemma~\ref{lem:stopping} yields the desired result.
  \end{proof}

\section{Proof of Lemma~\ref{lem:stability}}

\begin{proof}
  The first part of this proof is adapted from that of~\citet{ShalevBen-David14a}[Section~13.3.2] for smooth and nonnegative losses. Note that our bound does not assume the nonnegativity of the instantaneous loss. 
  \paragraph{Exact ERM} Denote by $Z^{(i)}$ the sample set that is identical to $Z$ except that the $i$-th sample $\xi_i$ is replaced by another random sample $\xi_i^\prime$, by $\hF^{(i)} (\w)$ the empirical objective defined using $Z^{(i)}$, \ie,
  \begin{align*}
    \hF^{(i)} (\w) := \frac{1}{b} \bigg( \sum_{j\neq i} \ell (\w, \xi_i) +  \ell (\w, \xi_i^\prime) \bigg) + r (\w),
  \end{align*}
  and by $\hw^{(i)}=\argmin_{\w}\ \hF^{(i)} (\w)$ the empirical risk minimizer of $\hF^{(i)} (\w)$.

  First, observe that 
  \begin{align} 
    \hF (\hw^{(i)}) - \hF (\hw) 
    & = \frac{\ell (\hw^{(i)}, \xi_i) - \ell (\hw, \xi_i) }{b} 
    + \frac{\sum_{j\neq i} \ell (\hw^{(i)}, \xi_j) - \ell (\hw, \xi_j) }{b} 
    + r (\hw^{(i)}) - r (\hw) \nonumber \\
    & = \frac{\ell (\hw^{(i)}, \xi_i) - \ell (\hw, \xi_i) }{b} 
    + \frac{ \ell (\hw, \xi_i^\prime) - \ell (\hw^{(i)}, \xi_i^\prime) }{b} 
    + \left( \hF^{(i)} (\hw^{(i)}) - \hF^{(i)} (\hw) \right) \nonumber  \\ % \label{e:replacement-stability-smoothness}
    \nonumber 
    & \le \frac{ \ell (\hw^{(i)}, \xi_i) - \ell (\hw, \xi_i) }{b} 
    + \frac{ \ell (\hw, \xi_i^\prime) - \ell (\hw^{(i)}, \xi_i^\prime) }{b}
    - \frac{\gamma - \sigma}{2} \norm{ \hw^{(i)} - \hw }^2
  \end{align}
  where we have used in the inequality the fact that $\hw^{(i)}$ is the minimizer of $\hF^{(i)} (\w)$, which is $(\gamma - \sigma)$-strongly convex.

  On the other hand, it follows from the $(\gamma - \sigma)$-strong convexity of $\hF (\w)$ that 
  \begin{align*} 
    \hF (\hw^{(i)}) - \hF (\hw) \ge \frac{\gamma - \sigma}{2} \norm{ \hw^{(i)} - \hw }^2 .
  \end{align*}

  Combining the above two inequalities, and applying the $\sigma$-almost convexity and $\beta$-smoothness of $\ell(\w,\xi)$, we obtain 
  \begin{align} 
    \; &  (\gamma - \sigma) b \cdot \norm{ \hw^{(i)} - \hw }^2 \nonumber \\ \label{e:stability-1}
    \le & \left( \ell (\hw^{(i)}, \xi_i) - \ell (\hw, \xi_i) \right) + \left( \ell (\hw, \xi_i^\prime) - \ell (\hw^{(i)}, \xi_i^\prime) \right)  \\
    \le & \dotp{\nabla \ell(\hw^{(i)},\xi_i)}{\hw^{(i)} - \hw} + \frac{\sigma}{2} \norm{\hw^{(i)} - \hw}^2 +  \dotp{\nabla \ell(\hw,\xi_i^\prime)}{\hw - \hw^{(i)}} + \frac{\sigma}{2} \norm{\hw - \hw^{(i)}}^2  \nonumber \\
    = & \dotp{\nabla \ell(\hw^{(i)},\xi_i) - \nabla \ell(\hw,\xi_i^\prime)}{\hw^{(i)} - \hw} + \sigma \norm{\hw^{(i)} - \hw}^2  \nonumber \\ 
    = & \dotp{\nabla \ell(\hw^{(i)},\xi_i) - \nabla \phi (\hw^{(i)})}{\hw^{(i)} - \hw} + 
    \dotp{\nabla \phi (\hw^{(i)}) - \nabla \phi (\hw)}{\hw^{(i)} - \hw} \nonumber \\
    & \quad
    + \dotp{\nabla \phi (\hw) - \nabla \ell(\hw,\xi_i^\prime)}{\hw^{(i)} - \hw}
    + \sigma \norm{\hw^{(i)} - \hw}^2  \nonumber \\ 
    \label{e:stability-2}
    \le & \rbr{\norm{\nabla \ell(\hw^{(i)},\xi_i) - \nabla \phi(\hw^{(i)})} + \norm{\nabla \ell(\hw,\xi_i^\prime) - \nabla \phi(\hw)}} \cdot \norm{\hw^{(i)} - \hw} + (\sigma + \beta) \norm{\hw^{(i)} - \hw}^2.
  \end{align}
  
  By the assumption that $\sigma + \beta \le  \frac{(\gamma - \sigma) b}{2}$, we then have 
  \begin{align*}
    \norm{\hw^{(i)} - \hw} & \le \frac{1}{(\gamma - \sigma) b - (\sigma + \beta)} \rbr{\norm{\nabla \ell(\hw^{(i)},\xi_i) - \nabla \phi(\hw^{(i)})} + \norm{\nabla \ell(\hw,\xi_i^\prime) - \nabla \phi(\hw)}} \\
    & \le \frac{2}{(\gamma - \sigma) b} \rbr{\norm{\nabla \ell(\hw^{(i)},\xi_i) - \nabla \phi(\hw^{(i)})} + \norm{\nabla \ell(\hw,\xi_i^\prime) - \nabla \phi(\hw)}}. 
  \end{align*}

  Taking expectations of~\eqref{e:stability-1} and~\eqref{e:stability-2} over the samples and plugging in the above inequality yields
  \begin{align*} 
    & \quad  2 \bbE_{Z \cup \cbr{\xi_i^\prime}} \left[ \phi (\hw) - \hphi (\hw) \right] \\
    & \le \rbr{ \frac{2}{(\gamma - \sigma) b} + \frac{4 (\sigma + \beta)}{(\gamma - \sigma)^2 b^2} }\cdot  \bbE_{Z \cup \cbr{\xi_i^\prime}} \rbr{\norm{\nabla \ell(\hw^{(i)},\xi_i) - \nabla \phi(\hw^{(i)})} + \norm{\nabla \ell(\hw,\xi_i^\prime) - \nabla \phi(\hw)}}^2 \\
    & \le \frac{4}{(\gamma - \sigma) b} \bbE_{Z \cup \cbr{\xi_i^\prime}} \rbr{\norm{\nabla \ell(\hw^{(i)},\xi_i) - \nabla \phi(\hw^{(i)})} + \norm{\nabla \ell(\hw,\xi_i^\prime) - \nabla \phi(\hw)}}^2 \\
    & \le \frac{8}{(\gamma - \sigma) b} \rbr{ \bbE_{Z^{(i)}} \sbr{ \bbE_{\cbr{\xi_i}} \norm{\nabla \ell(\hw^{(i)},\xi_i) - \nabla \phi(\hw^{(i)})}^2} + \bbE_{Z} \sbr{ \bbE_{\cbr{\xi_i^\prime}} \norm{\nabla \ell(\hw,\xi_i^\prime) - \nabla \phi(\hw)}^2} } \\
    & \le \frac{16 V^2}{(\gamma - \sigma) b}
  \end{align*}
  where we have used the triangle inequality in the first step, the assumption $2 (\sigma + \beta) \le (\gamma - \sigma) b$ in the second step, the fact that $(x + y)^2 \le 2 (x^2 + y^2)$ in the third step, and the assumption on the variance of stochastic gradient in the final step.

  Since $\hw$ minimizes $\hF (\w)$, we have $\hF (\hw) \le \hF (\w^*)$ and consequently 
  \begin{align*}
    \bbE_{Z} \sbr{F (\hw)}
    = \bbE_{Z} \sbr{\hF (\hw) + \phi(\hw) - \hphi(\hw)}
    \le F (\w^*) + \bbE_{Z} \sbr{\phi(\hw) - \hphi(\hw)} 
    \le F (\w^*) + \frac{8 V^2}{(\gamma - \sigma) b}
  \end{align*}
  Then the first part of the lemma follows. 
  \paragraph{Inexact ERM} For the approximate solution $\tw$, due to the $(\gamma - \sigma)$-strong convexity of $\hF (\w)$, we have
  \begin{align*}
    \bbE_{Z,\calA} \norm{\tw - \hw}^2 \le \frac{2}{\gamma - \sigma} \bbE_{Z,\calA} \left[ \hF(\tw) - \hF(\hw) \right] \le \frac{2 \delta}{\gamma - \sigma}.
  \end{align*}
  Now, in view of the $(\beta+\gamma)$-smoothness of $F(\w)$, it holds that 
  \begin{align*}
    \bbE_{Z,\calA} \sbr{F (\tw) - F (\hw)}
    & \le \bbE_{Z,\calA} \dotp{\nabla F (\hw)}{\tw - \hw} + \frac{\beta+\gamma}{2} \bbE_{Z,\calA} \norm{\tw - \hw}^2  \\
    & \le \bbE_{Z,\calA} \sbr{\norm{\nabla F (\hw)} \cdot \norm{\tw - \hw}} + \frac{\beta+\gamma}{2} \bbE_{Z,\calA} \norm{\tw - \hw}^2 \\
    & \le \bbE_{Z,\calA} \sbr{ \sqrt{2 (\beta + \gamma) \cdot (F(\hw) - F (\w^*))} \norm{\tw - \hw} } + \frac{\beta+\gamma}{2} \bbE_{Z,\calA} \norm{\tw - \hw}^2 \\
    & \le \bbE_{Z} \sbr{F(\hw) - F (\w^*)} + (\beta+\gamma) \cdot \bbE_{Z,\calA} \norm{\tw - \hw}^2 \\
    & \le  \frac{8 V^2}{(\gamma - \sigma) b} +  \frac{2 (\beta+\gamma) \delta}{\gamma - \sigma}
  \end{align*}
  where we have used~\citet[Theorem~2.1.5]{Nester04a} in the third inequality, and the fact that $xy \le x^2 + \frac{y^2}{4}$ in the fourth inequality. Then the lemma follows by combining this inequality with the stability of exact ERM. 
\end{proof}

\section{Proof of Theorem~\ref{thm:convergence}}

\begin{proof}
Let $\delta=\frac{8V^2}{(\beta+\gamma)b}$, then by the second part of Lemma~\ref{lem:stability}, it holds that 
\begin{align*}
\bbE_{Z_t, \calA} \sbr{F_t (\w_t) - F_t (\w_t^*)} \le \frac{32 V^2}{(\gamma - \sigma) b}. 
\end{align*}
Combining this with Corollary~\ref{cor:reduction} yields
\begin{gather*}
  \bbE_{R,\calA} \norm{\nabla \phi(\w_R)}^2  
   \le \frac{ 4 \gamma \rbr{\phi (\w_0) - \phi^*}}{K}  + (4 \beta + 8 \gamma) \cdot \frac{32 V^2}{(\gamma - \sigma) b} \\
  = \frac{ 4 \sigma \rbr{\phi (\w_0) - \phi^*}}{K} +  \frac{256 V^2}{b} +  \frac{ 4 (\gamma - \sigma) \rbr{\phi (\w_0) - \phi^*}}{K}   
  + (4 \beta + 8 \sigma) \cdot \frac{32 V^2}{(\gamma - \sigma) b}.
\end{gather*}
Minimizing the right hand side over $\gamma$ yields the optimal choice $\gamma = \sigma + \sqrt{\frac{32 (\beta + 2 \sigma) V^2 K}{\rbr{\phi (\w_0) - \phi^*} b}}$ and the desired result.
\end{proof}

\section{Proof of Lemma~\ref{lem:small-minibatch}}

\begin{proof} Denote by $\x^*=\argmin_{\x}\; F(\x)$ the unique minimizer of $F(\x)$, 
  and by $G_s (\x) = \phi(\x) + \frac{\gamma}{2} \norm{\x - \y}^2 + \frac{\rho_s}{2} \norm{\x - \x_{s-1}}^2$ the population counterpart of $\hG_s (\x)$, with unique minimizer $\x_s^*=\argmin_{\x} G_s (\x)$. In the following, we also use the shorthand $L = \beta + \gamma$, and $\lambda = \gamma - \sigma$. 

  First, by the $(\lambda+\rho_s)$-strong convexity of $\hG_s(\x)$, we have
  \begin{align} \label{e:mp-deterministic}
    \hG_s (\x^*) \ge \hG_s (\hx_s) + \frac{\lambda+\rho_s}{2} \norm{\x^* - \hx_s}^2. 
  \end{align}
  By the first part of Lemma~\ref{lem:stability} (we are now applying the lemma to $\hG_s (\x)$ and $G_s (\x)$, whose data-independent regularizer $\frac{\gamma}{2} \norm{\x - \y}^2 + \frac{\rho_s}{2} \norm{\x - \x_{s-1}}^2$ is $(\gamma+\rho_s)$-strongly convex), we have that $\bbE_{Z_s} \sbr{\phi (\x_s^*) - \frac{1}{m} \sum_{i=1}^m \ell(\x_s^*,\xi_i^s)} \le \frac{8 V^2}{(\lambda+\rho_s) m}$ as long as $m \ge \frac{2 (\sigma+\beta)}{\lambda +\rho_s}$. 
  Therefore, taking expectation of~\eqref{e:mp-deterministic} over $Z_s$ yields
  \begin{gather} 
    F (\x^*) + \frac{\rho_s}{2} \norm{\x^* - \x_{s-1}}^2 
    \ge \bbE_{Z_s} \sbr{ F (\hx_s) +  \frac{\rho_s}{2} \norm{\hx_s - \x_{s-1}}^2 - \frac{8 V^2}{(\lambda+\rho_s) m} + \frac{\lambda + \rho_s}{2} \norm{\x^* - \hx_s}^2} \nonumber \\ \label{e:mp-exact-generalization}
    \ge \bbE_{Z_s} \sbr{ F (\x_s) +  \frac{\rho_s}{2} \norm{\x_s - \x_{s-1}}^2 - \frac{16 V^2}{(\lambda+\rho_s) m} - \frac{2 (L + \rho_s) \eta_s}{\lambda + \rho_s}  + \frac{\lambda + \rho_s}{2} \norm{\x^* - \hx_s}^2}
  \end{gather}
  where we have used the second part of Lemma~\ref{lem:stability} in the second inequality.

  Next, we relate $\hx_s$ to $\x_s$ for the last term of~\eqref{e:mp-exact-generalization}. By the $(\lambda+\rho_s)$-strong convexity of $\hG_s(\x)$, we have 
  $\bbE_{Z_s,\calA} \norm{\hx_s - \x_s}^2 \le \frac{2 \eta_s}{\lambda+\rho_s}$, 
  and then by the triangle inequality
  \begin{align}
    \bbE_{Z_s,\calA} \norm{\x^* - \hx_s}^2 & \ge \bbE_{Z_s,\calA} \abs{ \norm{\x^* - \x_s} - \norm{\hx_s - \x_s} }^2  \nonumber \\
    & \ge \bbE_{Z_s,\calA} \norm{\x^* - \x_s}^2 - 2 \bbE_{Z_s,\calA} \sbr{ \norm{\x^* - \x_s} \cdot \norm{\hx_s - \x_s} } \nonumber \\
    & \ge \bbE_{Z_s,\calA} \norm{\x^* - \x_s}^2 - 2 \sqrt{ \bbE_{Z_s,\calA} \norm{\x^* - \x_s}^2} \sqrt{\bbE_{Z_s,\calA}\norm{\hx_s - \x_s}^2  } \nonumber  \\ \label{e:mp-inexact-1}
    & \ge \bbE_{Z_s,\calA} \norm{\x^* - \x_s}^2 - 2 \sqrt{ \bbE_{Z_s,\calA} \norm{\x^* - \x_s}^2} \sqrt{ \frac{2 \eta_s}{\lambda+\rho_s} }
  \end{align}
  where the third inequality is due to the Cauchy-Schwarz inequality.

  Substituting~\eqref{e:mp-inexact-1} into~\eqref{e:mp-exact-generalization} and rearranging terms, we obtain
  \begin{align*}
    \bbE_{Z_s,\calA} \sbr{ F (\x_s) - F (\x^*) } & \le \bbE_{Z_s,\calA} \sbr{ \frac{\rho_s}{2} \norm{\x^* - \x_{s-1}}^2 - \frac{\lambda + \rho_s}{2} \norm{\x^* - \x_s}^2} \\
    & \quad + \frac{16 V^2}{(\lambda+\rho_s) m} + \frac{2 (L + \rho_s) \eta_s}{\lambda + \rho_s} + \sqrt{ 2 (\lambda+\rho_s) \eta_s } \sqrt{ \bbE_{Z_s,\calA} \norm{\x^* - \x_s}^2}. 
  \end{align*}
  Setting $\rho_s=\frac{\lambda (s-1)}{2}$, and multiplying both sides by $s$, we further obtain
  \begin{align*}
    s \bbE_{Z_s,\calA} \sbr{ F (\x_s) - F (\x^*) } & \le \bbE_{Z_s,\calA} \sbr{ \frac{\lambda (s-1) s}{4} \norm{\x^* - \x_{s-1}}^2 - \frac{\lambda s (s+1)}{4} \norm{\x^* - \x_s}^2} \\
    & \quad + \frac{32 V^2}{\lambda m} + \frac{(4L + \lambda s) \eta_s}{\lambda} + \sqrt{  \lambda s \eta_s } \sqrt{ \bbE_{Z_s,\calA} \sbr{ s (s+1) \norm{\x^* - \x_s}^2} }. 
  \end{align*}

  Summing the above inequality over $s=1,\dots,S$ yields
  \begin{gather}
    \bbE \sbr{ \sum_{s=1}^S s \rbr{F (\x_s) - F (\x^*)} } + \bbE \sbr{\frac{S(S+1) \norm{\x^* - \x_S}^2}{4}} \nonumber \\ \label{e:mp-inexact-2}
    \le \frac{32 V^2 S}{\lambda m} + \frac{4L}{\lambda} \sum_{s=1}^S \eta_s + \sum_{s=1}^S s \eta_s 
    + \sum_{s=1}^S \sqrt{ \lambda s \eta_s } \sqrt{ \bbE \sbr{ s (s+1) \norm{\x^* - \x_s}^2} } .
  \end{gather}

  \paragraph{Bounding $\norm{\x^* - \x_S}^2$} Dropping the $\bbE \sbr{ \sum_{s=1}^S s \rbr{F (\x_s) - F (\x^*)} }$ term from~\eqref{e:mp-inexact-2} which is nonnegative, we have 
  \begin{gather*}
    \bbE \sbr{S(S+1) \norm{\x^* - \x_S}^2} \\ 
    \le \frac{128 V^2 S}{\lambda m} + \frac{16 L}{\lambda} \sum_{s=1}^S \eta_s + 4 \sum_{s=1}^S s \eta_s 
    + 4 \sum_{s=1}^S \sqrt{ \lambda s \eta_s } \sqrt{ \bbE \sbr{ s (s+1) \norm{\x^* - \x_s}^2} } .
  \end{gather*}
  Now apply Lemma~\ref{lem:resolve-recursion} and we obtain
  \begin{align*}
    \bbE \sbr{S(S+1) \norm{\x^* - \x_S}^2} \le 
    \sqrt{\frac{128 V^2 S}{\lambda m}} + \sqrt{\frac{16 L}{\lambda} \sum_{s=1}^S \eta_s} + \sqrt{4 \sum_{s=1}^S s \eta_s } + 4 \sum_{s=1}^S \sqrt{ \lambda s \eta_s }.
  \end{align*}
  Note that this bound increases with $S$.

  \paragraph{Bounding the function value} Dropping the $\bbE \sbr{S(S+1) \norm{\x^* - \x_S}^2}$ term from~\eqref{e:mp-inexact-2} which is nonnegative, we have 
  \begin{align}
    & \bbE \sbr{ \sum_{s=1}^S s \rbr{F (\x_s) - F (\x^*)} } \nonumber \\
    \le & \frac{32 V^2 S}{\lambda m} + \frac{4L}{\lambda} \sum_{s=1}^S \eta_s + \sum_{s=1}^S s \eta_s  +  \sum_{s=1}^S \sqrt{ \lambda s \eta_s } \sqrt{ \bbE \sbr{ S (S+1) \norm{\x^* - \x_S}^2}}   \nonumber \\
    \le & \frac{32 V^2 S}{\lambda m} + \frac{4L}{\lambda} \sum_{s=1}^S \eta_s + \sum_{s=1}^S s \eta_s \nonumber \\ \label{e:mp-inexact-3}
    & \qquad + \rbr{ \sum_{s=1}^S \sqrt{ \lambda s \eta_s } } \rbr{\sqrt{\frac{128 V^2 S}{\lambda m}} + \sqrt{\frac{16 L}{\lambda} \sum_{s=1}^S \eta_s} + \sqrt{4 \sum_{s=1}^S s \eta_s } + 4 \sum_{s=1}^S \sqrt{ \lambda s \eta_s }  }. 
  \end{align}

  We require that $\eta_s$ decays with $s$, and in particular 
  \begin{align} \label{e:inexact-strongly-convex-eta_t}
    \eta_s = \frac{V^2 S}{L m} \cdot \frac{1}{s^5}.
  \end{align}
  Recall that $\sum_{s=1}^{\infty} \frac{1}{s^{1+\delta}} \le \frac{1+\delta}{\delta}$. Then~\eqref{e:inexact-strongly-convex-eta_t} ensures
  \begin{align*}
    \sum_{s=1}^S \eta_s \le 2 \frac{V^2 S}{L m},\qquad
    \sum_{s=1}^S \sqrt{\lambda s \eta_s} \le 2 \sqrt{\frac{\lambda V^2 S}{L m}}, \qquad \text{and} \quad
    \sqrt{ \sum_{s=1}^S s \eta_s} \le 2 \sqrt{\frac{V^2 S}{L m}}.
  \end{align*}
  Plugging them into~\eqref{e:mp-inexact-3}, and noting that $\lambda \le L$, we obtain
  \begin{align*}
    \bbE \sbr{ \sum_{s=1}^S s \rbr{F (\x_s) - F (\x^*)} }
    \le  \frac{100 V^2 S}{\lambda m}.
  \end{align*}

  By returning the weighted average $\bar{\x}_S = \frac{2}{S (S+1)} \sum_{s=1}^S s \x_s$ and the convexity of $F(\x)$, we obtain the desired result.
\end{proof}

\section{An auxiliary lemma}

\begin{lem} 
  (\citealp[Lemma~1]{Schmid_11a}) \label{lem:resolve-recursion}
  Assume that the non-negative sequence $\cbr{u_S}$ satisfies the following recursion for all $S \ge 1$:
  \begin{align*}
    u_S^2 \le A_S + \sum_{s=1}^S \lambda_s u_s,
  \end{align*}
  with $A_S$ an increasing sequence, $A_0 \ge u_0^2$ and $\lambda_s \ge 0$ for all $s$. Then, for all $S \ge 1$, we have 
  \begin{align*}
    u_S \le \frac{1}{2} \sum_{s=1}^S \lambda_s + \left( A_S + \left(\frac{1}{2} \sum_{s=1}^S \lambda_s \right)^2 \right)^{\frac{1}{2}} \le \sqrt{A_S} + \sum_{s=1}^S \lambda_s.
  \end{align*}
\end{lem}

%\clearpage
%\input{lowerbound}
\bibliography{macp,macp-xref}

\begin{thebibliography}{38}
\providecommand{\natexlab}[1]{#1}
\providecommand{\url}[1]{\texttt{#1}}
\expandafter\ifx\csname urlstyle\endcsname\relax
  \providecommand{\doi}[1]{doi: #1}\else
  \providecommand{\doi}{doi: \begingroup \urlstyle{rm}\Url}\fi

\bibitem[Agarwal et~al.(2012)Agarwal, Bartlett, Ravikumar, and
  Wainwright]{Agarwal_12a}
Alekh Agarwal, Peter~L. Bartlett, Pradeep Ravikumar, and Martin~J. Wainwright.
\newblock Information-theoretic lower bounds on the oracle complexity of
  stochastic convex optimization.
\newblock \emph{IEEE Trans. Information Theory}, 58\penalty0 (5):\penalty0
  3235--3249, 2012.

\bibitem[Allen-Zhu(2017)]{Allen-Zhu17a}
Zeyuan Allen-Zhu.
\newblock Natasha: {Faster} non-convex stochastic optimization via strongly
  non-convex parameter.
\newblock In Doina Precup and Yee~Whye Teh, editors, \emph{Proc. of the 34rd
  Int. Conf. Machine Learning (ICML 2017)}, Sydney, Australia, August~6--11
  2017.

\bibitem[Bertsekas(1979)]{Bertsek79a}
D.~P. Bertsekas.
\newblock Convexification procedures and decomposition methods for nonconvex
  optimization problems.
\newblock \emph{J. Optimization Theory and Applications}, 29\penalty0
  (2):\penalty0 169--197, 1979.
\newblock URL \url{http://web.mit.edu/dimitrib/www/Convexification_Mult.pdf}.

\bibitem[Bertsekas(1999)]{Bertsek99a}
Dimitri~P. Bertsekas.
\newblock \emph{Nonlinear Programming}.
\newblock Athena Scientific, Nashua, NH, second edition, 1999.

\bibitem[Bertsekas(2015)]{Bertsek15a}
Dimitri~P. Bertsekas.
\newblock Incremental aggregated proximal and augmented {Lagrangian}
  algorithms.
\newblock arXiv:1509.09257 [cs.SY], November~4 2015.

\bibitem[Bottou(1991)]{Bottou91a}
L.~Bottou.
\newblock Stochastic gradient learning in neural networks.
\newblock In \emph{Proc. Neuron\^{\i}mes}, 1991.

\bibitem[Bousquet and Elisseeff(2002)]{BousquetElisseef02a}
Olivier Bousquet and Andr{\'e} Elisseeff.
\newblock Stability and generalization.
\newblock \emph{Journal of Machine Learning Research}, 2:\penalty0 499--526,
  March 2002.

\bibitem[Carmon et~al.(2017)Carmon, Duchi, Hinder, and Sidford]{Carmon_17a}
Yair Carmon, John~C. Duchi, Oliver Hinder, and Aaron Sidford.
\newblock Accelerated methods for non-convex optimization.
\newblock arXiv:1611.00756 [math.OC], February~2 2017.

\bibitem[Cotter et~al.(2011)Cotter, Shamir, Srebro, and Sridharan]{Cotter_11a}
Andrew Cotter, Ohad Shamir, Nati Srebro, and Karthik Sridharan.
\newblock Better mini-batch algorithms via accelerated gradient methods.
\newblock In J.~Shawe-Taylor, R.~S. Zemel, P.~Bartlett, F.~Pereira, and K.~Q.
  Weinberger, editors, \emph{Advances in Neural Information Processing Systems
  (NIPS)}, volume~24, pages 1647--1655. MIT Press, Cambridge, MA, 2011.

\bibitem[Crammer et~al.(2006)Crammer, Dekel, Keshet, Shalev-Shwartz, and
  Singer]{Crammer_06a}
Koby Crammer, Ofer Dekel, Joseph Keshet, Shai Shalev-Shwartz, and Yoram Singer.
\newblock Online passive-aggressive algorithms.
\newblock \emph{Journal of Machine Learning Research}, 7:\penalty0 551--585,
  2006.

\bibitem[Defazio(2016)]{Defazio16a}
Aaron Defazio.
\newblock A simple practical accelerated method for finite sums.
\newblock In Daniel~D. Lee, Ulrike von Luxburg, and Isabelle Guyon, editors,
  \emph{Advances in Neural Information Processing Systems (NIPS)}, volume~29.
  MIT Press, Cambridge, MA, 2016.

\bibitem[Dekel et~al.(2012)Dekel, Gilad-Bachrach, Shamir, and Xiao]{Dekel_12a}
Ofer Dekel, Ran Gilad-Bachrach, Ohad Shamir, and Lin Xiao.
\newblock Optimal distributed online prediction using mini-batches.
\newblock \emph{Journal of Machine Learning Research}, 13:\penalty0 165--202,
  2012.

\bibitem[Duchi et~al.(2011)Duchi, Hazan, and Singer]{Duchi_11a}
John Duchi, Elad Hazan, and Yoram Singer.
\newblock Adaptive subgradient methods for online learning and stochastic
  optimization.
\newblock \emph{Journal of Machine Learning Research}, 12:\penalty0 2121--2159,
  July 2011.

\bibitem[Ghadimi and Lan(2016)]{GhadimLan16a}
Saeed Ghadimi and Guanghui Lan.
\newblock Accelerated gradient methods for nonconvex nonlinear and stochastic
  programming.
\newblock \emph{Math. Prog.}, 156\penalty0 (1):\penalty0 59--99, 2016.

\bibitem[Ghadimi et~al.(2016)Ghadimi, Lan, and Zhang]{Ghadim_16a}
Saeed Ghadimi, Guanghui Lan, and Hongchao Zhang.
\newblock Mini-batch stochastic approximation methods for nonconvex stochastic
  composite optimization.
\newblock \emph{Math. Prog.}, 155\penalty0 (1--2):\penalty0 267--305, 2016.

\bibitem[Goyal et~al.(2017)Goyal, Doll{\'a}r, Girshick, Noordhuis, Wesolowski,
  Kyrola, Tulloch, Jia, and He]{Goyal_17a}
Priya Goyal, Piotr Doll{\'a}r, Ross Girshick, Pieter Noordhuis, Lukasz
  Wesolowski, Aapo Kyrola, Andrew Tulloch, Yangqing Jia, and Kaiming He.
\newblock Accurate, large minibatch {SGD}: {Training} imagenet in 1 hour.
\newblock arXiv:1706.02677 [cs.CV], June~8 2017.

\bibitem[Hoffer et~al.(2017)Hoffer, Hubara, and Soudry]{Hoffer_17a}
Elad Hoffer, Itay Hubara, and Daniel Soudry.
\newblock Train longer, generalize better: {Closing} the generalization gap in
  large batch training of neural networks.
\newblock arXiv:1705.08741 [stat.ML], May~24 2017.

\bibitem[Kingma and Ba(2015)]{KingmaBa15a}
Diederik Kingma and Jimmy Ba.
\newblock Adam: {A} method for stochastic optimization.
\newblock In \emph{Proc. of the 3rd Int. Conf. Learning Representations (ICLR
  2015)}, San Diego, CA, May~7--9 2015.

\bibitem[Kulis and Bartlett(2010)]{KulisBartlet10a}
Brian Kulis and Peter~L. Bartlett.
\newblock Implicit online learning.
\newblock In Johannes F{\"u}rnkranz and Thorsten Joachims, editors, \emph{Proc.
  of the 27th Int. Conf. Machine Learning (ICML 2010)}, pages 575--582, Haifa,
  Israel, June~21--25 2010.

\bibitem[Lan(2012)]{Lan12a}
Guanghui Lan.
\newblock An optimal method for stochastic composite optimization.
\newblock \emph{Math. Prog.}, 133\penalty0 (1--2):\penalty0 365--397, June
  2012.

\bibitem[{LeCun} et~al.(1998){LeCun}, Bottou, Orr, and M{\"u}ller]{Lecun_98b}
Yann {LeCun}, Leon Bottou, Genevieve~B. Orr, and Klaus-Robert M{\"u}ller.
\newblock Efficient backprop.
\newblock volume 1524 of \emph{Lecture Notes in Computer Science}, pages 9--50,
  Berlin, 1998. Springer-Verlag.

\bibitem[Lee et~al.(2016)Lee, Lin, Ma, and Yang]{Lee_16b}
Jason~D. Lee, Qihang Lin, Tengyu Ma, and Tianbao Yang.
\newblock Distributed stochastic variance reduced gradient methods and a lower
  bound for communication complexity.
\newblock arXiv:1507.07595 [math.OC], January~6 2016.

\bibitem[Li et~al.(2014)Li, Zhang, Chen, and Smola]{Li_14e}
Mu~Li, Tong Zhang, Yuqiang Chen, and Alexander~J. Smola.
\newblock Efficient mini-batch training for stochastic optimization.
\newblock In \emph{Proc. of the 20th ACM SIGKDD Int. Conf. Knowledge Discovery
  and Data Mining (SIGKDD 2014)}, pages 661--670, New York City, NY,
  August~24--27 2014.

\bibitem[Lin et~al.(2015)Lin, Mairal, and Harchaoui]{Lin_15a}
Hongzhou Lin, Julien Mairal, and Zaid Harchaoui.
\newblock A universal catalyst for first-order optimization.
\newblock In C.~Cortes, N.~D. Lawrence, D.~D. Lee, M.~Sugiyama, and R.~Garnett,
  editors, \emph{Advances in Neural Information Processing Systems (NIPS)},
  volume~28, pages 3366--3374. MIT Press, Cambridge, MA, 2015.

\bibitem[Loosli et~al.(2007)Loosli, Canu, and Bottou]{Loosli_07a}
Ga{\"e}lle Loosli, St{\'e}phane Canu, and L{\'e}on Bottou.
\newblock Training invariant support vector machines using selective sampling.
\newblock In L{\'e}on Bottou, Olivier Chapelle, Dennis {DeCoste}, and Jason
  Weston, editors, \emph{Large Scale Kernel Machines}, Neural Information
  Processing Series, pages 301--320. MIT Press, 2007.

\bibitem[Martens(2010)]{Marten10a}
James Martens.
\newblock Deep learning via {Hessian}-free optimization.
\newblock In Johannes F{\"u}rnkranz and Thorsten Joachims, editors, \emph{Proc.
  of the 27th Int. Conf. Machine Learning (ICML 2010)}, pages 735--742, Haifa,
  Israel, June~21--25 2010.

\bibitem[Nemirovski and Yudin(1983)]{NemirovYudin83a}
A.~S. Nemirovski and D.~B. Yudin.
\newblock \emph{Problem Complexity and Method Efficiency in Optimization}.
\newblock John Wiley \& Sons, 1983.

\bibitem[Nesterov(2004)]{Nester04a}
Y.~Nesterov.
\newblock \emph{Introductory Lectures on Convex Optimization. {A} Basic
  Course}.
\newblock Number~87 in Applied Optimization. Springer-Verlag, 2004.

\bibitem[Pearlmutter(1994)]{Pearlm94a}
Barak~A. Pearlmutter.
\newblock Fast exact multiplication by the {Hessian}.
\newblock \emph{Neural Computation}, 6\penalty0 (1):\penalty0 147--160, January
  1994.

\bibitem[Reddi et~al.(2016)Reddi, Konecny, Richtarik, Poczos, and
  Smola]{Reddi_16b}
Sashank~J. Reddi, Jakub Konecny, Peter Richtarik, Barnabas Poczos, and Alex
  Smola.
\newblock {AIDE}: {Fast} and communication efficient distributed optimization.
\newblock arXiv:1608.06879 [math.OC], August~24 2016.

\bibitem[Schmidt et~al.(2011)Schmidt, Roux, and Bach]{Schmid_11a}
Mark Schmidt, Nicolas~Le Roux, and Francis Bach.
\newblock Convergence rates of inexact proximal-gradient methods for convex
  optimization.
\newblock In J.~Shawe-Taylor, R.~S. Zemel, P.~Bartlett, F.~Pereira, and K.~Q.
  Weinberger, editors, \emph{Advances in Neural Information Processing Systems
  (NIPS)}, volume~24, pages 1458--1466. MIT Press, Cambridge, MA, 2011.

\bibitem[Shalev-Shwartz and Ben-David(2014)]{ShalevBen-David14a}
Shai Shalev-Shwartz and Shai Ben-David.
\newblock \emph{Understanding Machine Learning: {From} Theory to Algorithms}.
\newblock Cambridge University Press, 2014.

\bibitem[Shalev-Shwartz et~al.(2009)Shalev-Shwartz, Shamir, Srebro, and
  Sridharan]{Shalev_09a}
Shai Shalev-Shwartz, Ohad Shamir, Nathan Srebro, and Karthik Sridharan.
\newblock Stochastic convex optimization.
\newblock In Sanjoy Dasgupta and Adam Klivans, editors, \emph{Proc. of the 22th
  Annual Conference on Learning Theory (COLT'09)}, Montreal, Quebec,
  June~18--21 2009.

\bibitem[Shamir(2016)]{Shamir16c}
Ohad Shamir.
\newblock Without-replacement sampling for stochastic gradient methods:
  {Convergence} results and application to distributed optimization.
\newblock In Daniel~D. Lee, Ulrike von Luxburg, and Isabelle Guyon, editors,
  \emph{Advances in Neural Information Processing Systems (NIPS)}, volume~29,
  pages 46--54. MIT Press, Cambridge, MA, 2016.

\bibitem[Shamir et~al.(2014)Shamir, Srebro, and Zhang]{Shamir_14a}
Ohad Shamir, Nati Srebro, and Tong Zhang.
\newblock Communication-efficient distributed optimization using an approximate
  {Newton}-type method.
\newblock In Eric Xing and Tony Jebara, editors, \emph{Proc. of the 31st Int.
  Conf. Machine Learning (ICML 2014)}, pages 1000--1008, Beijing, China,
  June~21--26 2014.

\bibitem[Vapnik(2000)]{Vapnik00a}
Vladimir~N. Vapnik.
\newblock \emph{The Nature of Statistical Learning Theory}.
\newblock Springer Series in Information Science and Statistics.
  Springer-Verlag, Berlin, second edition, 2000.

\bibitem[Wang et~al.(2017)Wang, Wang, and Srebro]{Wang_17b}
Jialei Wang, Weiran Wang, and Nathan Srebro.
\newblock Memory and communication efficient distributed stochastic
  optimization with minibatch prox.
\newblock In Satyen Kale and Ohad Shamir, editors, \emph{Annual Conference on
  Learning Theory}, Amsterdam, Netherlands, July~7--10 2017.

\bibitem[Zeiler(2012)]{Zeiler12a}
Matthew~D. Zeiler.
\newblock {ADADELTA}: {An} adaptive learning rate method.
\newblock arXiv:1212.5701 [cs.LG], December 2012.

\end{thebibliography}

\end{document}